\documentclass{article}
\PassOptionsToPackage{numbers, compress}{natbib}
\usepackage[preprint]{neurips_2023}

\usepackage[utf8]{inputenc}

\usepackage{hyperref}
\hypersetup{
colorlinks=true,
linkcolor=blue,
filecolor=magenta,
urlcolor=cyan,
citecolor=cyan
}

\usepackage[utf8]{inputenc} %
\usepackage[T1]{fontenc}    %
\usepackage{hyperref}       %
\usepackage{url}            %
\usepackage{booktabs}       %
\usepackage{amsfonts}       %
\usepackage{nicefrac}       %
\usepackage{microtype}      %
\usepackage{xcolor} %
\usepackage{float}

\usepackage{xcolor}
\usepackage{mathtools,amsfonts,amssymb,mdframed,xspace,xfrac,multicol,bbm,bm}
\usepackage{makecell}
\usepackage{algorithm,algorithmic}
\usepackage{eucal}
\usepackage{amssymb,amsmath,amsopn,mathtools}
\usepackage{mathrsfs}
\usepackage{complexity}
\usepackage{enumitem}
\usepackage{subcaption}
\usepackage{dsfont}
\usepackage{wrapfig}
\usepackage{framed}
\usepackage{amsthm}
\usepackage[capitalize,noabbrev]{cleveref}

\newtheorem{theorem}{Theorem}[section]

\usepackage{amsmath}
\usepackage{amssymb}
\usepackage{mathtools}
\usepackage{amsthm}
\usepackage{caption}
\usepackage{subcaption}
\usepackage[capitalize,noabbrev]{cleveref}

\usepackage{amsmath,amsfonts,bm}

\def\eqref#1{equation~\ref{#1}}

\def\1{\bm{1}}

\def\rmA{{A}}

\def\rmX{{X}}
\def\rmY{{Y}}

\def\ermA{{\textnormal{A}}}

\def\ermX{{\textnormal{X}}}
\def\ermY{{\textnormal{Y}}}

\DeclareMathAlphabet{\mathsfit}{\encodingdefault}{\sfdefault}{m}{sl}
\SetMathAlphabet{\mathsfit}{bold}{\encodingdefault}{\sfdefault}{bx}{n}

\usepackage{url}
\usepackage{bigints}

\usepackage{multirow}
\usepackage{graphicx,xcolor}
\usepackage{wrapfig}
\usepackage[normalem]{ulem}
\usepackage{wasysym,enumitem}

\def\PS{{\mathbb{P}^{S}}}
\def\PT{{\mathbb{P}^{T}}}
\def\Xcal{{\mathcal{X}}}
\def\Ycal{{\mathcal{Y}}}
\def\preds{{\textnormal{$\hat{\ermY}$}}}
\def\Acal{{\mathcal{A}}}
\newcommand{\FUNC}{{\mathcal{F}}}
\newcommand{\funcnew}{{\mathsf{F}}}

\newcommand{\ie}{\textnormal{i.e.}}
\def\TR{{\mathcal{D}^{S}}}
\def\TE{{\mathcal{D}^{T}}}

\def\ERS{{\widehat{\mathsf{ER}}^S}}
\def\ERT{{\widehat{\mathsf{ER}}^T}}
\def\TRS{{\mathcal{R}^S}}
\def\TRT{{\mathcal{R}^T}}
\def\L{{\mathcal{L}}}
\def\Wass{{\mathcal{W}}}

\def\Metric{{\mathcal{M}}}

\def\EOdds{{\Delta_{\mathrm{EOdds}}}}

\def\Apar{{\Delta_{\mathrm{Apar}}}}

\theoremstyle{plain}
\theoremstyle{definition}
\newtheorem{definition}[theorem]{Definition}
\newtheorem{assumption}[theorem]{Assumption}
\theoremstyle{remark}

\title{Fairness under Covariate Shift:
\\ Improving Fairness-Accuracy tradeoff with few Unlabeled Test Samples}

\author{%
Shreyas Havaldar \thanks{Equal Contribution} \\ Google Research India \\
  \texttt{shreyasjh@google.com} \\ \And Jatin Chauhan \footnotemark[1] \textsuperscript{    ,} \thanks{Contributions to this work were made when affiliated with Google Research India} \\ UCLA \\
  \texttt{chauhanjatin100@gmail.com} \\ \And Karthikeyan Shanmugam \footnotemark[1] \\ Google Research India \\
  \texttt{karthikeyanvs@google.com} \\ \And Jay Nandy \footnotemark[2] \\ Fujitsu Research India \\
  \texttt{jayjaynandy@gmail.com} \\ \And Aravindan Raghuveer \\ Google Research India \\
  \texttt{araghuveer@google.com}
}

\begin{document}

\maketitle

\begin{abstract}
Covariate shift in the test data is a common practical phenomena that can significantly downgrade both the accuracy and the fairness performance of the model. Ensuring fairness across different sensitive groups under covariate shift is of paramount importance due to societal implications like criminal justice. We operate in the unsupervised regime where only a small set of unlabeled test samples along with a labeled training set is available. Towards improving fairness under this highly challenging yet realistic scenario, we make three contributions. First is a novel composite weighted entropy based objective for prediction accuracy which is optimized along with a representation matching loss for fairness. We experimentally verify that optimizing with our loss formulation outperforms a number of state-of-the-art baselines in the pareto sense with respect to the fairness-accuracy tradeoff on several standard datasets. Our second contribution is a new setting we term Asymmetric Covariate Shift that, to the best of our knowledge, has not been studied before. Asymmetric covariate shift occurs when distribution of covariates of one group shifts significantly compared to the other groups and this happens when a dominant group is over-represented. While this setting is extremely challenging for current baselines, We show that our proposed method significantly outperforms them. Our third contribution is theoretical, where we show that our weighted entropy term along with prediction loss on the training set approximates test loss under covariate shift. Empirically and through formal sample complexity bounds, we show that this approximation to the unseen test loss does not depend on importance sampling variance which affects many other baselines.

\end{abstract}

\section{Introduction}

Predictions of machine learnt models are used to make important decisions that have societal impact, like in criminal justice, loan approvals, to name a few.
Therefore, there is a lot of interest in understanding,  analyzing and improving  model performance along other dimensions like robustness~\citep{silva2020opportunities}, model generalization~\citep{wiles2021fine} and fairness~\citep{oneto2020fairness}. 
In this work, we focus on the algorithmic fairness aspect.  Datasets used for training could be biased in the sense that some groups may be under-represented, thus biasing classifier decisions towards the over-represented group or the bias could be in terms of undesirable causal pathways between sensitive attribute and the label in the real world data generating mechanism~\citep{oneto2020fairness}. It has often been observed~\citep{bolukbasi2016man}, ~\citep{buolamwini2018gender} that algorithms that optimize predictive accuracy that are fed pre-existing biases further learn and then propagate the same biases.
Improving fairness of learning models has received significant attention from the research community \cite{mitchell2021algorithmic}.

Another common challenge that models deployed in real world situations face is that of {\em Covariate Shift.}
In covariate shift,  the distribution of covariates (feature vectors) across training and testing changes, however the optimal label predictor conditioned on input remains the same. Therefore the model may make wrong predictions when deployed or more seriously can slowly degrade over time when the covariate shift is gradual. Due to the practical importance of this problem, there has been a significant amount of research in detecting covariate shift and modeling methodologies to address it \cite{wilson2020survey,redko2020survey}.

The problem that we study in this paper is at the juncture of the above two hard problems: ensuring fairness under covariate shift.   While this question has not received much attention, some recent works like \cite{rezaei2021robust} have begun to address this problem. We also introduce a new variant of covariate shift called {\em Asymmetric covariate shift}  where  distribution of covariates of one group shifts significantly as compared to the other groups. Asymmetric covariate shift is a very common practical situation when there is long tail of underrepresented groups in the training data. For example, consider the popular and important task of click through prediction of advertisements~\cite{li2015click}.  Small and medium sized advertisers have poorer representation in the training data because they do not spend as much as the large businesses on  advertising. Therefore during inference Small and Medium (SMB) advertisement clicks will see significantly more co-variate shift compared to those clicks on ads from large advertisers.
Also, due to the nature of the problem of covariate shift, access to large labeled test is often not possible. In summary, the problem we aim to tackle is 
"Provide a high fairness-accuracy tradeoff under both symmetric and asymmetric covariate shift while having acesss to a very small set of unlabeled test samples". To this end, we make three key contributions in this paper.

 \noindent
    {1. We introduce a composite objective to approximate the prediction loss on the unlabeled test that involves \textit{a novel weighted entropy objective on the set of unlabeled test samples} along with ERM objective on the labeled training samples. We optimize these weights using \textit{min-max} optimization that implicitly drives these weights to importance sampling ratios with no density estimation steps. We show that our proposed objective has \textit{provably} lower variance compared to the importance sampling based methods. This composite objective is then combined with a representation matching loss to train fair classifiers. (Section~\ref{sec:method}). %
    }
    
\noindent {2. We introduce a new type of covariate shift called \textit{asymmetric covariate shift} wherein one protected group exhibits large covariate shift while the other does not. We highlight that fairness-accuracy tradeoff degrades under this case for existing methods (Section~\ref{sec:Existing-Methods}). We show empirically that the combination of our objective and representation matching achieves the best accuracy fariness-tradeoff even in this case.  %
}

\noindent {3. By incorporating our proposed weighted entropy objective with the Wasserstein based representation matching across sub-groups, we empirically compare against a number of baseline methods on benchmark datasets. In particular, we achieve the best accuracy-equalized odds tradeoff in the \textit{pareto sense}}. %

\section{Related Work}
\label{sec:related_work}

\textbf{Techniques for imposing fairness:} \textit{Pre-processing} techniques aim to transform the dataset \citep{calmon2017optimized, swersky2013learning,
kamiran2012data} followed by a standard training. \textit{In-processing} methods directly modify the learning algorithms using techniques, such as, adversarial learning \citep{madras2018learning, 10.1145/3278721.3278779}, \citep{agarwal2018reductions, JMLR:v20:18-616,donini2018empirical, fish2016confidence, zafar2017fairness, celis2019classification}. \textit{Post-processing} approaches, primarily focus on modifying the outcomes of the predictive models in order to make unbiased predictions \citep{pleiss2017fairness, 
hardt2016equality}

\textbf{Distribution Shift:} Research addressing distribution shift in machine learning is vast and is growing. The general case considers a joint distribution shift between training and testing data \citep{ben2006analysis, blitzer2007learning, moreno2012unifying} resulting in techniques like domain adaptation \citep{ganin2015unsupervised}, distributionally robust optimization \citep{sagawa2019distributionally,duchi2021learning} and invariant risk minimization and its variants \citep{arjovsky2019invariant,
shi2021gradient}. A survey of various methods and their relative performance is discussed by \cite{wiles2021fine}. We focus on the problem of \textit{Covariate Shift} where the \textit{Conditional Label} distribution is invariant while there is a shift in the marginal distribution of the covariates across training and test samples. %
This classical setup is studied by \cite{shimodaira2000improving, sugiyama2007covariate, gretton2009covariate}.
\textit{Importance Weighting} is one of the prominently used techniques for tackling covariate shifts \citep{sugiyama2007covariate,https://doi.org/10.48550/arxiv.1910.06324}. 
However, they are known to have high variance under minor shift scenarios \citep{NIPS2010_59c33016}. Recently methods that  emerged as the de-facto approaches to tackle distribution shifts include popular entropy minimization \citep{wang2021tent}, pseudo-labeling \citep{adaptation2_2017,adaptation3_cvpr_2020}, batch normalization adaptation \citep{NEURIPS2020_85690f81,adaptBN_cp_arxiv_2020}, because of their wide applicability and superior performance. Our work provides a connection between a version of weighted entropy minimization and traditional importance sampling based loss which may be of independent interest. 
\begin{table*}[t]
\centering
\resizebox{1\textwidth}{!}{%
\begin{tabular}{|l|l|l|l|}
\hline 
Method        & Metrics Optimized For & Labels Required & Method Description \\ \hline
Adv-Deb. ~\citep{10.1145/3278721.3278779} & Eq. Odds              & Yes                   & Representation Matching (across data subgroups grouped by (A,Y))                           \\ \hline
Adv-Deb. ~\citep{10.1145/3278721.3278779} & Demographic Parity             & No                   & Representation Matching (across data subgroups grouped by only A)                           \\ \hline
FairFictPlay ~\cite{kearns2018preventing}        & False Positive Rate Parity              & Yes                   & Minimax Game between learner and auditor                           \\ \hline
LFR ~\citep{swersky2013learning}          & Demographic Parity     & No                    & Representation Matching                          \\ \hline
RSF*  ~\citep{rezaei2021robust}         & Eq. Odds              & No                   & Minimax Game between predictor and test distribution approximator                           \\ \hline
Massaging \cite{kamiran2012data}    & Demographic Parity                     & Yes                    &            Changing class labels to remove discrimination               \\ \hline
Suppression \cite{kamiran2012data}           & Demographic Parity                     & Yes                    & Remove sensitive attributes along with other highly correlated ones                          \\ \hline
Reweighing \cite{kamiran2012data}           & Demographic Parity                     & Yes                    & Tuples in training dataset are assigned weights                          \\ \hline
Sampling \cite{kamiran2012data}           & Demographic Parity                     & Yes                    & Non-uniform sampling of groups via duplication and removal                         \\ \hline
RF*   ~\citep{NEURIPS2020_d6539d3b}         & Demographic Parity                     &    No                & Weighted Combination of Dataset via 2 Player Game                         \\ \hline
Opt. Pre-Processing ~\citep{calmon2017optimized}         & Demographic Parity                     &    Yes                & Data Transformation                       \\ \hline
Certification and Combinatorial Repairing   ~\citep{feldman2015certifying}         & Disparate Impact                     &    Yes                & Creating a new distribution via linear interpolation in rank space                      \\ \hline
Certification and Geometric Repairing   ~\citep{feldman2015certifying}         & Disparate Impact                     &    Yes                & Creating a new distribution via linear interpolation in original dataspace                        \\ \hline
LAFTR ~\citep{madras2018learning}        & Demographic Parity                    &    No                & Adversarial Representation Matching                     \\ \hline
LAFTR ~\citep{madras2018learning}        & Eq. Odds                    &    Yes                & Adversarial Representation Matching                     \\ \hline
LAFTR ~\citep{madras2018learning}        & Eq. Opportunity                    &    Yes                & Adversarial Representation Matching                     \\ \hline
EGR ~\citep{agarwal2018reductions}        & Demographic Parity                   &    No                &    Exp. Gradient Reduction               \\ \hline
EGR ~\citep{agarwal2018reductions}        & Eq. Odds                   &    Yes                &    Exp. Gradient Reduction                 \\ \hline

Post Processing ~\citep{hardt2016equality}        & Eq. Odds                   &    Yes                &    Modifying an existing predictor based on A and Y.              \\ \hline
\end{tabular}%
}
\caption{Representative Collection of Fairness Methods. * represents methods that tackle covariate shift that we compare with}
\label{tab:methods}
\end{table*}

\textbf{Fairness under Distribution shift:}  The work by~\cite{rezaei2021robust} is by far the most aligned to ours as they propose a method that is robust to covariate shift while ensuring fairness when unlabeled test data is available.
However, this requires the density estimation of training and test distribution that is not efficient at higher dimensions and small number of test samples. In contrast our method avoids density estimation and uses a weighted version of entropy minimization that is constrained suitably to reflect importance sampling ratios implicitly. \cite{NEURIPS2020_d6539d3b} proposed a method for fair classification under the worst-case weighting of the data via an iterative procedure, but it is in the agnostic setting where test data is not available. \cite{singh2021fairness} studied fairness under shifts through a causal lens but the method requires access to the causal graph, separating sets and other non-trivial data priors. \cite{zhang2021farf} proposed FARF, an adaptive method for learning in an online setting under fairness constraints, but is clearly different from the static shift setting considered in our work. \cite{slack2020fairness} proposed a MAML based algorithm to learn under fairness constraints, but it requires access to labeled test data. \cite{https://doi.org/10.48550/arxiv.2206.12796} propose a consistency regularization technique to ensure fairness under subpopulation and domain shifts under a specific model, while we consider covariate shift.

\section{Problem Setup}
\label{sec:problem_setup}
Let $\Xcal \subseteq \mathcal{R}^d$  be the $d$ dimensional feature space for covariates, $\Acal$ be the space of categorical \textit{group} attributes and $\Ycal$ be the space of class labels. In this work, we consider $\Acal = \{0, 1\}$ and $\Ycal = \{0, 1\}$. Let $\rmX \in \Xcal, \ermA \in \Acal, ~\ermY \in \Ycal$ be realizations from the space. We consider a training dataset $\TR = \{(\rmX_i,\ermA_i,\ermY_i) | i \in [n]\}$ where every tuple $(\rmX_i,\ermA_i,\ermY_i) \in \Xcal \times \Acal \times \Ycal$. We also have an \textit{unlabeled} test dataset, $\TE = \{\rmX_i,\ermA_i | i \in [m]\}$. We focus on the setup where $m << n$.  The training samples $(\rmX_i,\rmA_i,\rmY_i \in \TR)$ are sampled i.i.d from distribution $\PS(\rmX,\ermY,\ermA)$ while the unlabeled test instances are sampled from $\PT(\rmX,\ermA)$. 

Let $\FUNC: \Xcal \rightarrow [0,1]$ be the space of soft prediction models. In this work, we will consider $\funcnew \in \FUNC$ of the form $\funcnew= h \circ g$ where $g(\rmX) \in \mathbb{R}^{k}$ (for some dimension $k>0$), is a representation that is being learnt while $h(g(\rmX)) \in [0,1]$ provides the soft prediction. Note that we don't consider $\ermA$ as an input to $\funcnew$, as explained in the work of  ~\citep{https://doi.org/10.48550/arxiv.2106.08812}. 
$\funcnew$ is assumed to be parametrized via $\theta$. Instead of representing the network as $\funcnew_{\theta}$, we drop the subscript and simply use $\funcnew$ when its clear from the context. The class prediction probabilities from $\funcnew$ are denoted with $P(\preds=y|\rmX_i)$, where $y \in \{0, 1\}$.

The supervised in-distribution training of $\funcnew$ is done by minimizing the \emph{empirical risk}, $\ERS$ as the proxy for \emph{population risk}, $\TRS$. Both risk measures are computed using the \textit{Cross Entropy (CE)} loss for classification (correspondingly we use $\ERT$ and $\TRT$ over the \textit{test distribution} for $\funcnew$). %
\begin{align}
 & \TRS = \mathbb{E}_{\PS(\rmX,\ermA,\rmY)} \left( -\log P(\preds=\rmY|\rmX) \right), \nonumber  \\ &\ERS = \frac{1}{n} \sum_{(\rmX_i,\rmY_i,A_i) \in \TR} \left(- \log P(\preds=\rmY_i|\rmX_i) \right)
 \label{eq:train_risk}
\end{align}

\subsection{Covariate Shift Assumption}
\label{subsec:covar_shift}
For our work, we adopt the \textit{covariate shift} assumption as in \cite{shimodaira2000improving}. Covariate shift assumption implies that $\PS(\ermY|\rmX, \ermA) = \PT(\ermY|\rmX, \ermA)$. In other words, shift in distribution only affects the joint distribution of covariates and sensitive attribute, i.e. $\PS(\rmX, \ermA) \neq \PT(\rmX, \ermA)$. We note that our setup is identical to a recent work of fairness under covariate shift by \cite{rezaei2021robust}. We also define and focus on a special case of covariate shift called \hbox{\textit{asymmetric covariate shift}}.

\begin{definition}[Asymmetric Covariate Shift]\label{defn:asymmcov_shift}
 Asymmetric covariate shift occurs when distribution of covariates of one group shifts while the other does not, i.e. $\PT(\rmX | \ermA =1 ) \neq \PS(\rmX | \ermA =1 ) $ while $\PT(\rmX | \ermA = 0 ) = \PS(\rmX | \ermA =0 ) $ in addition to $\PS(\ermY|\rmX, \ermA) = \PT(\ermY|\rmX, \ermA)$
\end{definition}

\begin{figure}[t]
    \centering
    \begin{subfigure}[b]{\columnwidth}
        \centering
        \includegraphics[width=0.42\textwidth]{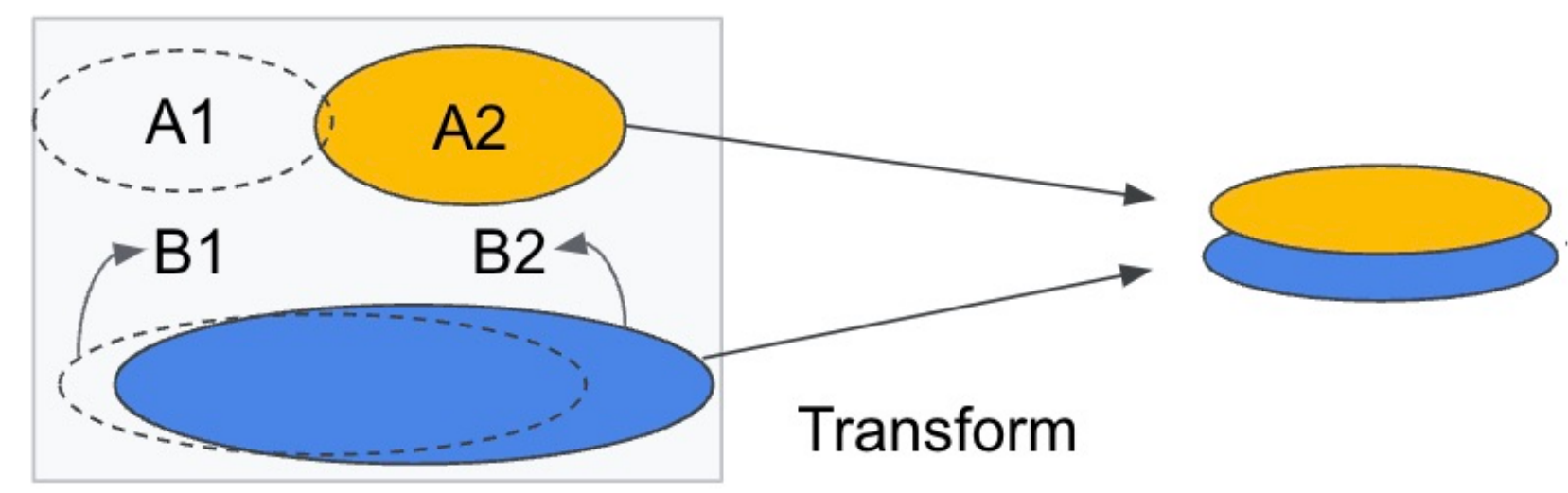}
        \caption{Figure shows how feature matching transformation affects domain of features for the two groups (yellow and the blue). The test features $B2,A2$ of both groups are made to overlap in the transformed space. }
        \label{fig:setup-1}
    \end{subfigure}
    \hfill
    \begin{subfigure}[b]{\columnwidth}
        \centering
        \includegraphics[width=0.42\textwidth]{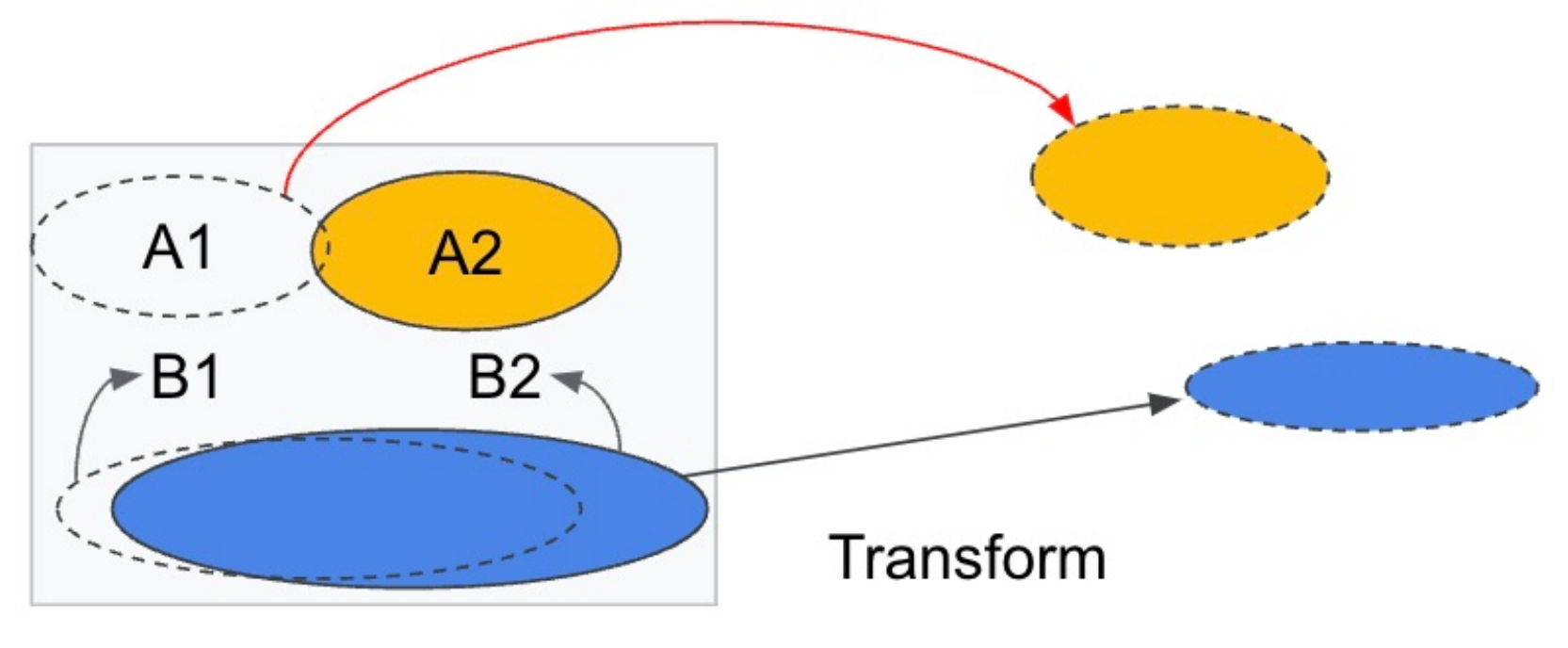}
        \caption{This shows how training features $A1,B1$ misalign in the transformed space when the same transformation in Figure \ref{fig:setup-1} is applied. Since $B1$ and $B2$ shift very little, in the transformed space, in the test region only the labeled examples of blue group is found. }
        \label{fig:setup-2}
     \end{subfigure}
     \caption{Asymmetric Shift Illustrated}
     \label{fig:setup-both}
\end{figure}

This type of covariate shift occurs when a sub-group is over represented (sufficiently capturing all parts of the domain of interest in the training data) while the other sub-group being under represented and observed only in one part of the domain. In the test distribution, covariates of the under-represented group assume a more drastic shift.

\subsection{Fairness Regularization with no Labels}
\label{sec:motivation}

We observe that the central issue in our problem is the learner has access only to an unlabelled test set ${\cal D}^T$ and one wants some fairness criterion to be enforced on it. Popular fairness metrics like \textit{Equalized Odds} (that forces $I(A; \hat{Y} | Y)$ to zero) and \textit{Accuracy Parity} (that forces $I(Y \neq \hat{Y};A)=0$)  depend on the training labels $Y$ to even evaluate. %

In Table \ref{tab:methods}, we outline fairness metrics optimized, methods used and popular methods in the fairness literature. We emphasise that this is by no means an exhaustive survey of fairness methods and only a representative one. Excluding methods that address accuracy-fairness  tradeoff under covariate shift (we compare against these baselines empirically in our work), we observe that only methods that use representation matching of covariates across sensitive groups can work to impose any fairness measure at all without requiring any label information. We also point out that representation matching can also help ensure accuracy parity under some conditions \cite{zhao2019inherent} without requiring labels. Therefore, we adopt representation matching loss across sensitive groups applied on ${\cal D}^T$ to be our fairness regularizer in this work.

Formally, we seek to train a classifier $\funcnew_\theta=h_\theta \circ g_\theta (\rmX)$ by matching representation $g(\rmX)$ across the protected sub groups and learning a classifier on top of that representation \citep{zhao2019inherent}. 
 Several variants for representation matching loss have been proposed in the literature ~\citep{pmlr-v115-jiang20a, https://doi.org/10.48550/arxiv.2103.06828,https://doi.org/10.48550/arxiv.2106.08812, NEURIPS2020_51cdbd26}. For implementation ease, we pick Wasserstein-2 metric to impose representation matching. We recall the definition of Wasserstein distance:
 \begin{definition}
		Let $(\Metric,d)$ be a metric space and $P_p(\Metric)$ denote the collection of all probability measures $\mu$ on $\Metric$ with finite $p^{th}$ moment. %
		Then the $p$-th Wasserstein distance between measures $\mu$ and $\nu$ both $\in P_p(\Metric)$ is given by:
	$\Wass_p(\mu,\nu) = \left(  \inf_\gamma \int_{\Metric \times \Metric} d(x,y)^{p} d\gamma(x,y) \right)^{\frac{1}{p}}$;
		$\gamma \in \Gamma(\mu,\nu)$, where $\Gamma(\mu,\nu) $ %
		denotes the collection of all measures on $\Metric \times \Metric$ with marginals $\mu$ and $\nu$ respectively.
\end{definition}

We minimize the $\Wass_2$ between the representation $g(\cdot)$ of the test samples from both groups. %
Empirically, our representation matching loss is given by:
 $   {\hat{\L}_{Wass}}(\TE) = \Wass_p(\hat{\mu},\hat{\nu}), \hat{\mu} = \frac{\sum_{(\ermX_i,\rmA_i=0)\in \TE} \delta_{g(\rmX_i)}}{|{(\ermX_i,\rmA_i=0)\in \TE}|} , \hat{\nu} = \frac{\sum_{(\ermX_i,\rmA_i=1)\in \TE} \delta_{g(\rmX_i)}}{|{(\ermX_i,\rmA_i=1)\in \TE}|} 
 $

We arrive at the following objective which is of central interest in the paper:
\begin{align}\label{eq:obj_interest}
\min \limits_{F_{\theta} = h_\theta \circ g_\theta } \ERT + \lambda {\hat{\L}_{Wass}}(\TE) 
\end{align}

\subsection{Issues with Applying Existing Techniques}\label{sec:Existing-Methods}
\textbf{Representation Matching:}
Since we don't have labels for the test set one cannot implement the first term in (\ref{eq:obj_interest}). It is natural to optimize the following objective: $\ERS + \lambda {\hat{\L}_{Wass}}(\TE) $ where the first term optimizes prediction accuracy on labeled training data while the second term matches representation across groups in the unlabeled test. We illustrate that under asymmetric covariate shift, this above objective is ineffective. The issue is illustrated best through Figures \ref{fig:setup-1} and \ref{fig:setup-2}. $A1$ and $B1$ represent group $1$ and $0$ feature distributions in the training set. Under asymmetric covariate shift, $B2 \approx B1$ while group $1$ shifts drastically to $A2$. Now, representation matching loss on the test would map $A2$ and $B2$ to the same region in the range space of $g(\cdot)$ as in Fig. \ref{fig:setup-1}. However, the classifier $h$ would be exclusively trained on samples from group $B$ (i.e. $g(B1)$) although both $A2$ and $B2$ overlap there as shown in Fig. \ref{fig:setup-2}. Training predictors on training samples but representation matching under test creates this issue as discussed in detail in appendix.
This highlights also the central issue we tackle in this work. Adversarial debiasing \cite{10.1145/3278721.3278779} is another method that does representation matching which we discuss below.

\textbf{Under-Performance of Adversarial Debiasing under Covariate Shift}
\label{sec:ecs-adult}
We study the variation of the performance of a State of The Art Method, namely Adversarial Debiasing~\citep{10.1145/3278721.3278779} against the magnitude of shift $\gamma$ on the Adult dataset. The variation of the error \% is plotted in the left subfigure of Figure~\ref{fig:ecs-adult}, and the variation of Equalized Odds is present in the right subfigure.

As we move from a setting of No Covariate Shift ($\gamma=0$), to Medium Covariate Shift ($\gamma = 10$), to Extreme Covariate Shift ($\gamma = 20$), there is a significant deterioration in the performance of the method as both Error \% and Equalized Odds increase with the increase in the magnitude of the shift. This reinforces the belief that state of the art fairness techniques do not extend well to settings under covariate shift.
In figure~\ref{fig:ecs-adult}, we complement these claims by analyzing the \textit{under-performance} for a state-of-the-art fairness method - Adversarial Debiasing~\citep{10.1145/3278721.3278779}. We also see similar drop in performance under covariate shift in other baselines we consider, which we have highlighted in our experimental analysis.

\begin{figure}[t]
    \centering
    \includegraphics[width=0.9\columnwidth]{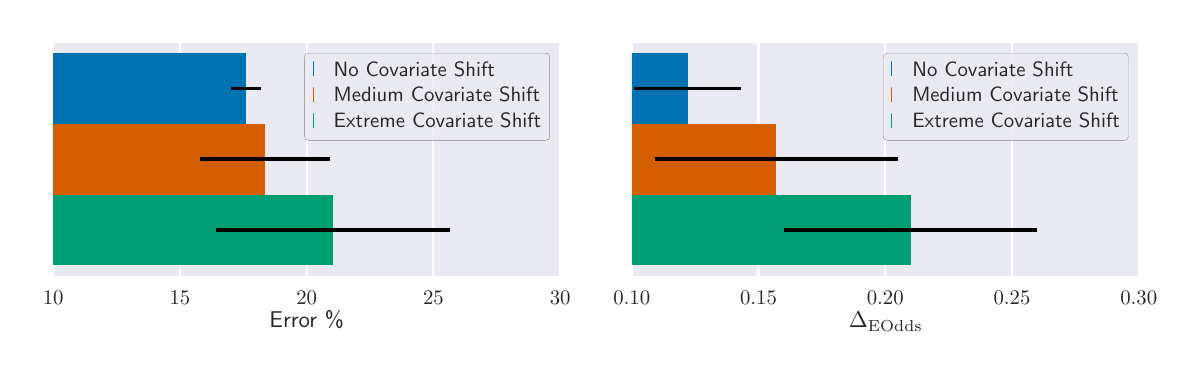}

    \caption{Both Error (in \% left) and Equalized Odds (right) for SOTA fairness method - Adversarial Debiasing exhibit strong degradation on increasing the magnitude of covariate shift. Three scenarios corresponding to no shift, intermediate shift and high shift are plotted (details on shift construction are explained in the experimental section). }

    \label{fig:ecs-adult}
\end{figure}

\textbf{Distributional Robustness Methods:} Another option to implement (\ref{eq:obj_interest}) would be to use a distributional robust learner (DRO) on the source distribution simultaneously with the representation matching penalty for the target. We consider a very recent SOTA method RGD-Exp \cite{kumar2023stochastic} that implements a form of DRO. We effectively replace $\ERT$ from eqn (\ref{eq:obj_interest}) with a robust loss term from the paper and perform the same optimization as us and notice that it does not achieve as good a accuracy-fairness tradeoff as our algorithm, thus establishing that trivially combining a SOTA distributionally robust method with Wasserstein Loss (2nd term from eqn. \ref{eq:obj_interest}: ${\hat{\L}_{Wass}}(\TE))$ does not suffice to achieve fairness under shift and something more nuanced is required. 

\textbf{Importance Sampling/Density Ratio Estimation based methods:} Another way to implement (\ref{eq:obj_interest}) is to use importance sampled prediction loss on training samples to mimic the test loss (first term) in (\ref{eq:obj_interest}). For this, one estimates ratio between training and test density directly using KLIEP/LSIF losses (\cite{sugiyama2007covariate,JMLR:v10:kanamori09a}) or perform density estimation which does not scale in higher dimensions. Sample complexity of these techniques directly scales with importance sampling variance which is large with very few test samples. We show this via formal sample complexity bounds in Section \ref{sec:theory} and empirically in 
figure \ref{fig:asymmetric_fig_results_1} and subsequent figures in appendix,
where we see large variances in accuracy for these methods. Robust Shift Fair Method of \cite{rezaei2021robust} also involves density estimation steps which suffer from the same issue.

\section{Method and Algorithm}
\label{sec:method}

Recall that the objective we are interested in is (\ref{eq:obj_interest}). One needs a proxy for the first term ($\ERT$) due to lack of labels. From considerations in the previous section, training has to be done in a manner that exploits training labels from source dataset effectively but can tackle covariate shift despite using representation matching.
We derive a novel objective in Theorem \ref{theorem:th_1} based on the weighted entropy over instances in $\TE$ along with empirical loss over $\TR$ and show that is an upper bound to $\TRT$. %

\begin{theorem}
\label{theorem:th_1}
Suppose that $\PT(\cdot)$ and $\PS(\cdot)$ are absolutely continuous with respect to each other over domain $\Xcal$. Let $\epsilon \in \mathbb{R}^{+}$ be such that $ \frac{\PT(\ermY=y|\rmX)}{P(\preds=y|\rmX)}  \leq \epsilon$, for $y \in \{0,1\}$ almost surely with respect to distribution $\PT(\rmX)$. %
Then, we can upper bound $\TRT$ using $\TRS$ along with an unsupervised objective over $\PT$ as:
\begin{align}
\label{eq:main_bound}
    \TRT \leq \TRS + \epsilon \times \mathbb{E}_{\PT(\rmX)} \left[ e^{\left(- \frac{\PS(\rmX)}{\PT(\rmX)}\right)} \mathcal{H(\preds|\rmX)} \right] %
\end{align}
where $\mathcal{H(\preds|\rmX)} = \sum_{y \in \{0,1\}} -P(\preds=y|\rmX) \log (P(\preds=y|\rmX))$ is the conditional entropy of the label given a sample $\rmX$. 
\end{theorem}
\begin{proof}%
The proof is relegated to the appendix.
\end{proof}

\textbf{Note}: Sections marked as A.x, B.y and C.z refer to sections in the supplementary section.
The mild assumption on $\epsilon$  in the theorem is also justified via extensive experiments in appendix.

We \textit{emphasize} that this result also provides an important connection and a rationale for using entropy based objectives as an unsupervised adaptation objective from an importance sampling point of view that has been missing in the literature ~\citep{wang2021tent, https://doi.org/10.48550/arxiv.1909.13231}. 

Entropy objective is imposed on points that are more typical with respect to the test than the training. Conversely, in the region where samples are less likely with respect to the test distribution, since it has been optimized for label prediction as part of training, the entropy objective is not imposed strongly. The above bound however hinges on the assumption that pointwise in the domain $\Xcal$, $\funcnew$ approximates the true soft predictor by at most a constant factor $\epsilon$. To ensure a small value of $\epsilon$, we resort to pre-training $\funcnew$ with only $\TR$ samples for a few epochs before imposing any other type of regularization. %

\subsection{Theoretical Analysis}\label{sec:theory}
The most widely used objective to optimize for (Left Hand Side) L.H.S of (\ref{eq:main_bound}), i.e. $\TRT$, leverages \textit{importance sampling} \citep{sugiyama2007covariate}, which we denote as $\mathcal{R}_{IS}$ here for clarity. We denote R.H.S of (\ref{eq:main_bound}) by $\mathcal{R}_{WE}$. Our method is motivated by the R.H.S of (\ref{eq:main_bound}). Here, we compare the generalization bounds for $\mathcal{R}_{IS}$ and  $\mathcal{R}_{WE}$ . We make the following assumptions to simplify the analysis as the task is to compare $\mathcal{R}_{IS}$ against $\mathcal{R}_{WE}$ only, however some of these can be relaxed trivially.

\begin{assumption}\label{assump}
\begin{itemize}[itemsep=-3pt]
\item Let $\Theta= \{ \theta_1  \ldots \theta_k\}$ be finite parameter space.
\item Let the losses $l_1(\cdot) = -log(P_{\theta}(\preds=\rmY|\rmX))$  and $l_2(\cdot) = \sum_{y \in \{0,1\} } - P_{\theta}(\hat{Y}=y|\rmX)  \log P_{\theta}(\hat{Y}=y|\rmX)$ be bounded between $[0,1]$ in the domain $\{0,1\} \times {\cal X}$ for all $\theta \in \Theta$. This is not a heavy assumption and can be achieved via appropriate Lipschitz log loss over bounded domain.
\item  Denoting the important weights $z(\rmX) = \frac{\PT(\rmX)}{\PS(\rmX)}$ (assuming we have access to exact importance weights), let $\sup \limits_{\rmX \in {\cal X}} z(\rmX) = M$ and the variance of the weights with respect to the training distribution be $\sigma^2$ .%
\end{itemize}
\end{assumption}

For the $\mathcal{R}_{IS}$ objective, we have the following the result,
\begin{theorem}\label{thm:IS_main}
Under Assumption \ref{assump}, with probability $1- \delta$ over the draws of $\TR \sim \PS$,  we have $\forall \theta \in \Theta$:
    $\mathbb{E}_{\PS}[\mathcal{R}_{IS}(\theta)] \leq  \widehat{\mathcal{R}}_{IS}(\theta) + \frac{2 M (\log \lvert \Theta \rvert + \log (1/\delta)) }{3 \lvert \TR \rvert} + \nonumber   \sqrt{2\sigma^2 \frac{(\log \lvert \Theta \rvert+  \log  (1/\delta)) }{ \lvert \TR \rvert} } $
\end{theorem}

Whereas for our objective $\mathcal{R}_{WE}$ (posing $\epsilon$ as a hyperparameter $\lambda$),
\begin{theorem} \label{thm:WE_main}
Under Assumption \ref{assump}, we have that with probability $1-2 \delta$ over the draws of $\TR \sim \PS$ and $\TE \sim \PT$, we have $\forall \theta \in \Theta$
    $\mathbb{E}_{\PS,\PT}[\mathcal{R}_{WE}(\theta)] \leq \widehat{\mathcal{R}}_{WE}(\theta) + 2  \sqrt{\frac{ 2\log \lvert \Theta \rvert}{ \lvert \TR \rvert} } + \nonumber 
     2 \lambda \sqrt{\frac{ 2\log \lvert \Theta \rvert}{ \lvert \TE \rvert} }  +  3  \sqrt{ \frac{\ln(2/\delta)}{2 \lvert \TR \rvert}} +  3 \lambda \sqrt{ \frac{\ln(2/\delta)}{2 \lvert \TE \rvert}}$
\end{theorem}

The proofs can be found in appendix.
Comparing Theorem \ref{thm:IS_main} and Theorem \ref{thm:WE_main}, we see that the generalization bound for importance sampled objective, $\mathcal{R}_{IS}$, depends on variance of importance weights $\sigma^2$ and also the worst case value $M$. In contrast, our objective, $\mathcal{R}_{WE}$, \textit{does not} depend on these parameters and thus does not suffer from high variance. These results are further justified empirically in section \ref{sec:main_results_}.  %

\subsection{Weighted Entropy Objective}
\label{subsec:weighted-ent-obj}
Implementing the objective in (\ref{eq:main_bound}), requires computation of %
$\frac{\PS(\rmX)}{\PT(\rmX)}$. This is  challenging when $m$ (amount of unlabeled test samples) is small and typical way of density estimation in high dimensions is particularly hard. Therefore, we propose to estimate the ratio $\frac{\PS(\rmX)}{\PT(\rmX)}$ by a parametrized network $\funcnew_w : \Xcal \rightarrow \mathbb{R}$, where $\funcnew_w(\rmX)$ shall satisfy the following constraints: $ \mathbb{E}_{\rmX \sim \PT(\rmX)}[\funcnew_w(\rmX)]=1, ~\mathrm{~and~} \mathbb{E}_{\rmX \sim \PS(\rmX)}[1/(\funcnew_w(\rmX))]=1 $.
By definition, these constraints must be satisfied.

Building on (\ref{eq:main_bound}), we solve for the following upper bound in Theorem ~\ref{theorem:th_1}:
\begin{align}
\label{eq:worst_case_ent}
    &\max_{\theta(\funcnew_w)}  \TRS +  \epsilon \times \mathbb{E}_{\PT(\rmX)} \left[ e^{\left(- \funcnew_w(\rmX)\right)} \mathcal{H(\preds|\rmX)} \right] \nonumber \\ ~ &\mathrm{s.t.~} \mathbb{E}_{\rmX \sim \PT(\rmX)}[ \funcnew_w(\rmX)]=1, ~\mathbb{E}_{\rmX \sim \PS(\rmX)}[1/(\funcnew_w(\rmX))]=1
\end{align}

\textbf{Final Learning Objective:} Finally, we plug in the empirical risk estimator for $\TRS$, approximate the expectation in second term with the empirical version over $\TE$, posit $\epsilon$ as a hyperparameter and add the unfairness objective  ${\hat{\L}_{Wass}}(\TE)$ as in (\ref{eq:obj_interest}). Furthermore, we utilize the output of the representation layer $g$ (denoting $F= h \circ g$, where $g$ is the encoder subnetwork and $h$ is the classifier subnetwork, 
as input to $\funcnew_w$ rather than raw input $\rmX$ (provable benefits of representation learning \citep{https://doi.org/10.48550/arxiv.1706.04601}). The optimization objective thus becomes:
\begin{align}
 & \min_{\funcnew_{\theta}} \max_{\funcnew_w} \L(\funcnew_{\theta},\funcnew_w) = \ERS + \nonumber \\
  &\lambda_1 \frac{1}{m} \sum_{\rmX_i \in \TE} \left[ e^{(-\funcnew_w(g(\rmX_i)))} \mathcal{H(\preds|\rmX)} \right] + \lambda_2{\hat{\L}_{Wass}}(\TE) \nonumber \\ 
&\mathrm{s.t.~} \mathcal{C}_1 = \frac{1}{m} \sum_{\rmX_i \in \TE} \funcnew_w(g(\rmX_i)) = 1 , ~\mathrm{~and~}\nonumber\\ &\mathcal{C}_2 = \frac{1}{n} \sum_{\rmX_i \in \TR} \frac{1}{\funcnew_w(g(\rmX_i))} = 1
\label{eq:entropy_final_obj}
\end{align}
Here $\lambda_1$ and $\lambda_2$ are hyperparameters governing the objectives. $\mathcal{C}_1$ and $\mathcal{C}_2$ refer to the constraints. We use alternating gradient updates to solve the above min-max problem. Our entire learning procedure consists of \textit{two stages}: (1) pre-training $\funcnew$ for some epochs with only $\TR$ and (2) further training $\funcnew$ with (\ref{eq:entropy_final_obj}). The procedure is summarized in Algorithm~\ref{our_algorithm} and a high level architecture is provided in appendix.

\begin{algorithm}[t]
    \caption{Gradient Updates for the proposed objective to learn fairly under covariate shift}
    \label{our_algorithm}
    
    \begin{algorithmic}
   \STATE {\bfseries Input:} Training data $\TR$, Unlabelled Test data $\TE$, model $\funcnew$, weight estimator $\funcnew_w$, decaying learning rate $\eta_t$, number of pre-training steps $\tilde{\mathcal{E}}$, number of training steps $\mathcal{E}$ for eq ~\ref{eq:entropy_final_obj}, $\lambda_1, \lambda_2$
    \STATE {\bfseries Output:} Optimized parameters $\theta^{*}$ of the model $\funcnew$
    \STATE
    \STATE $\theta^{0} \gets$ random initialization 
    \FOR{$t\gets1$  {\bfseries to}  $\tilde{\mathcal{E}}$}
    \STATE   $\theta^{t} \gets \theta^{t-1} - \eta_t \nabla_{\theta^{t-1}} \ERS$ 
    \ENDFOR
    \STATE $w^{\tilde{\mathcal{E}}} \gets$ random initialization 
    \FOR{$t\gets\tilde{\mathcal{E}}+1$ {\bfseries to} $\mathcal{E}+\tilde{\mathcal{E}}$}
     \STATE $w^{t} \gets w^{t-1}$ + $\eta_t \nabla_{w^{t-1}}\L(\theta^{t-1},w^{t-1}) $ ; \hspace{0.3cm} /* subject to $\mathcal{C}_1$ and $\mathcal{C}_2$  */
     \STATE $\theta^{t} \gets \theta^{t-1} - \eta_t \nabla_{\theta}\L(\theta^{t-1},w^{t})$; \hspace{0.3cm} /* gradient stopping is applied through $\funcnew_w$ in this step*/
    \ENDFOR
   \STATE $\theta^{*} \gets \theta^{\mathcal{E}+\tilde{\mathcal{E}}}$
    \end{algorithmic}
\end{algorithm}

\section{Experiments}
\label{sec:experiments}

We demonstrate our method on 4 widely used benchmarks in the fairness literature, i.e. Adult, Communities and Crime, Arrhythmia and Drug Datasets with detailed description in appendix.
The baseline methods used for comparison are: MLP, Adversarial Debias (AD)~\citep{10.1145/3278721.3278779}, Robust Fair (RF)~\citep{NEURIPS2020_d6539d3b}, Robust Shift Fair (RSF)~\citep{rezaei2021robust}, Z-Score Adaptation (ZSA). Along these, we also compare against two popular Density ratio estimation techniques, \cite{NIPS2007_be83ab3e} (KLIEP) and \cite{JMLR:v10:kanamori09a} (LSIF), that estimate the ratio $\frac{\PT(\rmX)}{\PS(\rmX)}$ via a parametrized setup. The estimates are then used to compute the \textit{importance weighted} training loss $\mathcal{R}_{IS}$ described previously. \cite{pmlr-v48-menon16} analysed both these methods in a unifying framework. We further provide comparisons against an adapted RGD-Exp \cite{kumar2023stochastic} as described in section \ref{sec:Existing-Methods} to demonstrate simply adapting DRO methods would not suffice. The detailed description for all the baselines is provided in appendix.
These baselines also cover the important works highlighted in Section~\ref{sec:related_work}.

The implementation details of all the methods with relevant hyperparameters are provided in appendix. 
The evaluation of our method against the baselines is done via the trade-off between fairness violation (using $\EOdds$) and error (which is $100 -$ accuracy). All algorithms are run $50$ times before reporting the mean and the standard deviation in the results. All experiments are run on single NVIDIA Tesla V100 GPU.

Apart from the primary results on standard and asymmetric shift below, extensive analyses across multiple settings are provided in appendix (due to space limitations).
\subsection{Shift Construction}
\label{subsec:shift_construction}
To construct the covariate shift in the datasets, \ie, to introduce $\PS(\rmX, \ermA) \neq \PT(\rmX, \ermA)$, we utilize the following strategy akin to the works of~\cite{rezaei2021robust, 10.7551/mitpress/9780262170055.003.0008}. First, all the non-categorical features are normalized by \emph{z-score}. We then obtain the \emph{first principal component} of the of the covariates and further project the data onto it, denoting it by $\mathcal{P_{C}}$. We assign a score to each point $\mathcal{P_{C}}[i]$ using the density function $\Xi : \mathcal{P_{C}}[i] \rightarrow {\rm e}^{\gamma \cdot (\mathcal{P_{C}}[i] - b)} / \mathcal{Z}$. Here, $\gamma$ is a hyperparameter controlling the level of distribution shift under the split, $b$ is the $60^{th} \permil$ (percentile) of $\mathcal{P_{C}}$ and $\mathcal{Z}$ is the normalizing coefficient computed empirically. Using this, we sample $40\%$ instances from the dataset as the test and remaining $60\%$ as training. To construct the validation set, we further split the training subset to make the final train:validation:test ratio as $5:1:4$, where the test is distribution shifted. Similar procedure is used to construct the shifts for asymmetric analysis in section \ref{subsec:app_aymm_section}.

Note that for large values of $\gamma$, all the points with $\mathcal{P_{C}}_{[i]} > b$ will have high density thereby increasing the probability of being sampled into the test set. This generates a sufficiently large distribution shift. Correspondingly, for smaller values of $\gamma$, the probability of being sampled is not sufficiently high for these points thereby leading to higher overlap between the train and test distributions.

\subsection{Fairness-Accuracy Tradeoff}
\label{sec:main_results_}
\begin{figure*}[t]
    \centering
    \includegraphics[width=1\textwidth]{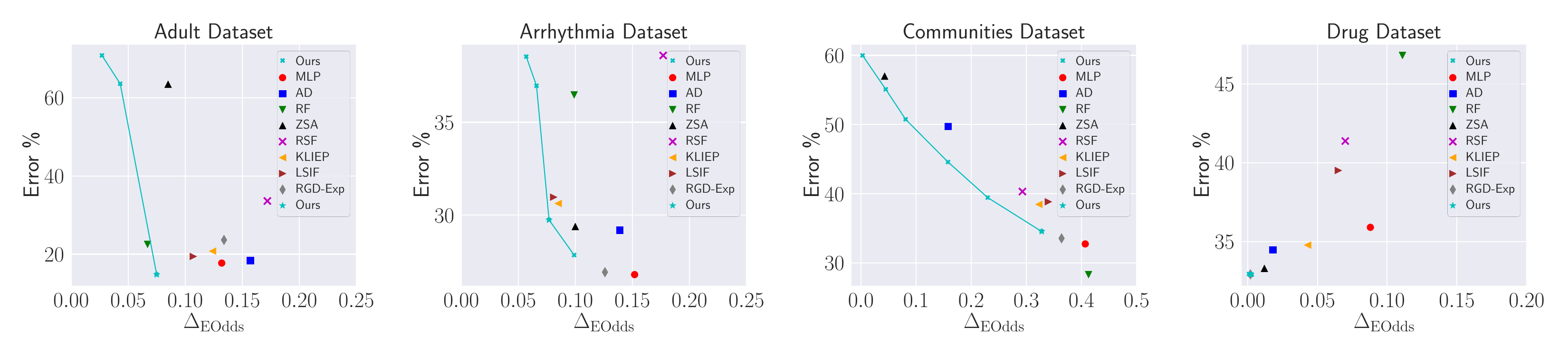}
    \caption{Fairness-Error Tradeoff Curves for our method (Pareto Frontier) against the optimal performance of the baselines. Our method provides better tradeoffs in all cases. (On Drug dataset, the performance is concentrated around the optimal point). All figures best viewed in colour.}
    \label{fig:symmetric_fig_results_pareto}
\end{figure*}
\begin{figure*}[t]
    \centering
    \includegraphics[width=1\textwidth]{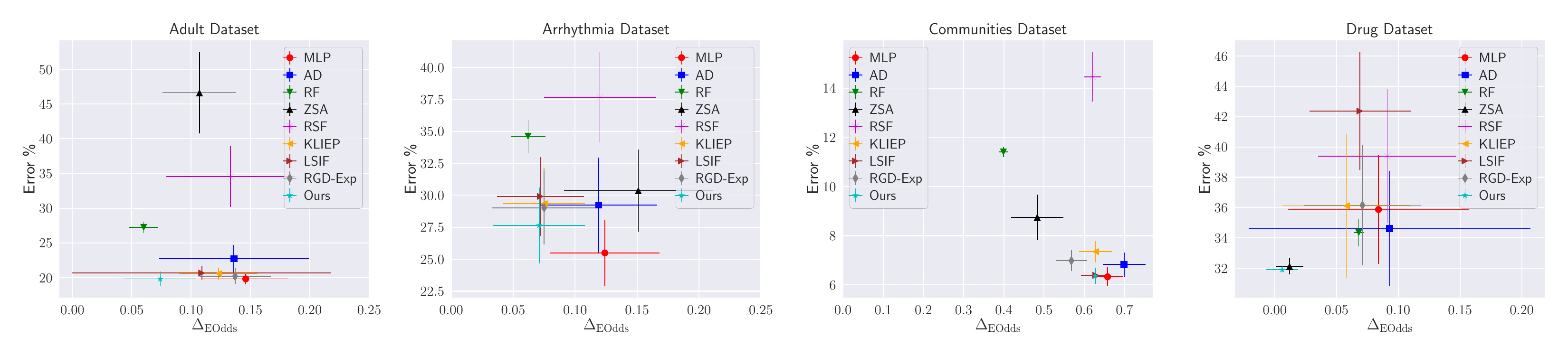}
    \caption{Comparison of our method against the baselines under Asymmetric Covariate Shift for group $\ermA = 0$.}
    \label{fig:asymmetric_fig_results_1}
\end{figure*}
The experimental results for the shift constructed using procedure in section~\ref{subsec:shift_construction} are shown in Figure \ref{fig:symmetric_fig_results_pareto}. The results closer to the \textit{bottom left} corner in each plot are desirable. 

Our method provides better error and fairness tradeoffs against the baselines on all the benchmarks. For example, on the Adult dataset, we have the lowest error rate at around 15\% with $\EOdds$ at almost $0.075$ while the closest baselines MLP and RF fall short on either of the metrics. On Arrhythmia and Communities, our method achieves very low $\EOdds$ (best on Arrhythmia with a margin of $\sim 30\%$) with only marginally higher error as compared to MLP and RF respectively. On the Drug dataset, we achieve the best numbers for both the metrics. For the same accuracy, we obtain 1.3x-2x improvements against the baselines methods on most of the benchmarks. Similarly for the same $\EOdds$, we achieve up to 1.5x lower errors. It is also important to note that all the other unsupervised adaptation algorithms perform substantially worse and are highly unreliable. For example, ZSA performs well only on the Drug dataset, but shows extremely worse errors (even worse than \textit{random predictions}) on Communities and Adult. The adaptation performed by ZSA is insufficient to handle covariate shift. RSF baseline is consistently worse across the board. This is because it tries to explicitly estimate $\PS(\rmX)$ and $\PT(\rmX)$ which is extremely challenging whereas we implicitly estimate the importance ratio. %

For KLIEP and LSIF, we equip both of these with the Wasserstein penalty term to provide a fair comparison.
First, we observe that our method consistently outperforms these algorithms across the datasets. The relative improvement of our method is as high as $\sim 31\%$ in error on Adult dataset and $\sim 32.5\times$ in $\EOdds$ on Drug dataset against LSIF. Similar non-trivial margins can be noted on other datasets. Second, as highlighted in 
the appendix, the variance in error rates of the KLIEP and LSIF based importance is very high on the Drug dataset. Particularly, both KLIEP and LSIF exhibit up to $20-40$ times higher variance in error and up to $10-12$ times in $\EOdds$. We can attribute this difference to the phenomenon that in the small sample regime, importance weighted objective on training dataset alone may not bring any improvements for covariate shift due to variance issues and thus estimating the ratio can be insufficient.

\textbf{Variance Details:} 
The detailed plots with variance corresponding to figure \ref{fig:symmetric_fig_results_pareto} are provided in appendix. In some cases, the standard deviation bars in the figure stretch beyond $0$ in $\mathbb{R}^{-}$ due to skewness when error bars are plotted, however numbers across all the runs are \textit{positive}. Low variance results of our method are notable, as discussed in section \ref{sec:main_results_} especially against KLEIP and LSIF.
\subsection{Results on Asymmetric Shift}
\label{subsec:app_aymm_section}
Here, we present empirical results for Asymmetric Covariate Shift where the degree of shift is substantially different across the groups. %
To construct data for this setting, we follow the same procedure as described in section~\ref{subsec:shift_construction}, but operate on data for the two groups differently. The shift is introduced in one of the groups while for the other group, we resort to  splitting it randomly into train-val-test.

Figure~\ref{fig:asymmetric_fig_results_1} provides the results for the setup when shift is created in group $\ermA = 0$. We again observe that our method provides better tradeoffs across the board. For the shift in group $\ermA = 0$, we have substantially better results on Adult and Arrhythmia with up to $\sim$ 2x improvements on $\EOdds$ for similar error and up to $\sim$ 1.4x improvements in error for similar $\EOdds$. On the Communities dataset, MLP and AD show similar performance to ours, but much worse on the Drug dataset for both the metrics. ZSA performs comparably to our method only on Drug, but is substantially worse on other datasets. This confirms the inconsistency of the baselines under this setup as well. %
The results for shift in group $\ermA = 1$ 
(relegated to Appendix), show analogous trends. 
On the Drug dataset we are \textbf{10x} and \textbf{5x} better than the two importance sampling baselines on $\EOdds$ without the significant variance and lower error \% as well. Even on other datasets we notice strong trends for our method with lower error and lower $\EOdds$ across the board. This shows that our method performs well in the Asymmetric Covariate Shift setting against importance sampling methods. It is also important to note that the errors are lower for all the methods as compared to figure~\ref{fig:symmetric_fig_results_pareto} since only one group exhibits substantial shift while degradation in equalized odds is higher. This is in line with the reasoning provided in section~\ref{sec:motivation} based on theorem C.1 in appendix.

\section{Conclusion}
In this work, we considered the problem of unsupervised test adaptation under covariate shift to achieve good fairness-error trade-offs using a small amount of unlabeled test data. We proposed a composite loss, that apart from prediction loss on training, involves a representation matching loss along with weighted entropy loss on the unsupervised test. We experimentally demonstrated the efficacy of our formulation. Our source code is made available for additional reference. \footnote{https://github.com/google/uafcs}

\clearpage

\bibliographystyle{plainnat}
\bibliography{bib_file}

\clearpage
\appendix

\section{Appendix}

\subsection{Proofs}
\label{sec:proofs}
\begin{proof}[Proof of Theorem \ref{theorem:th_1}]
We start with rewriting the expected cross entropy loss on the test as importance sampled loss on the training distribution.

\begin{align}
    \TRT &= \mathbb{E}_{\PT(\rmX)} \left\{ \sum_{y \in \{0,1\}} -\PT(\ermY=y|\rmX) \log (P(\preds=y|\rmX)) \right\} \\ 
    & \overset{a}{=} \mathbb{E}_{\PS(\rmX)}  \{ \left( \frac{\PT(\rmX)}{\PS(\rmX)} \right) \cdot \nonumber \\ &\sum_{y \in \{0,1\}} -\PS(\ermY=y|\rmX) \log (P(\preds=y|\rmX)) \}
\end{align}

($a$) is the \textbf{Importance Weighting} technique proposed by~\citep{sugiyama2007covariate} %
\begin{align}
    & = \TRS + \mathbb{E}_{\PS(\rmX)} \left( \frac{\PT(\rmX)}{\PS(\rmX)} - 1 \right) \cdot\nonumber \\ &\left\{ \sum_{y \in \{0,1\}} -\PS(\ermY=y|\rmX) \log (P(\preds=y|\rmX)) \right\} \\
    & = \TRS + \mathbb{E}_{\PT(\rmX)} \left( 1 - \frac{ \PS(\rmX)}{\PT(\rmX)} \right) \cdot\nonumber \\ &\left\{ \sum_{y \in \{0,1\}} -\PT(\ermY=y|\rmX) \log (P(\preds=y|\rmX)) \right\}
\end{align}
\begin{align}
    & \overset{b} {\leq}\TRS + \epsilon \times \mathbb{E}_{\PT(\rmX)}  \{ \left( 1 - \frac{ \PS(\rmX)}{\PT(\rmX)} \right)  \cdot\nonumber \\ & \sum_{y \in \{0,1\}} -P(\preds=y|\rmX) \log (P(\preds=y|\rmX)) \} \\
    & \overset{c}{\leq} \TRS + \epsilon \times \mathbb{E}_{\PT(\rmX)} \left[ e^{\left(- \frac{\PS(\rmX)}{\PT(\rmX)}\right)} \mathcal{H(\preds|\rmX)}  \right],  \label{eq:semifinal_ent} %
\end{align}
(b) is because of the assumption that $\frac{\PT(\ermY|\rmX)}{ P(\preds|\rmX) } \leq \epsilon $ almost surely with respect to $\rmX \sim \PT$.

(c) This is because $ 1 - x \leq e^{-x},~ x \geq 0$.
\end{proof}

\subsection{Dataset Description}
\label{subsec:dataset_description}
The detailed description of the datasets used in this work are as follows:
\begin{itemize}
    \item \textbf{Adult} is a dataset from the UCI repository containing details of individuals. The output variable is the indicator of whether the adult makes over \$50k a year. The group attribute is gender. Following~\cite{NEURIPS2020_d6539d3b}, we use the processed data with $2213$ examples and $97$ features.
    \item \textbf{Arrhythmia} is a dataset from the UCI repository where each example is classified between the presence and absence of cardiac arrhythmia. The group attribute is gender. Following~\cite{rezaei2021robust}, we used the dataset used containing 452 examples and 279 features.
    \item \textbf{Communities and Crime} is a dataset from the UCI repository where each example represents a community. The output variable is the community having a violent crime rate in the $70^{th}$ percentile of all the communities. The group attribute is the binary indicator of the presence of the majority white population. Following~\cite{NEURIPS2020_d6539d3b}, we use the dataset with $2185$ examples and $122$ features.
    \item \textbf{Drug} is a dataset from the UCI repository where the task is to classify the type of drug consumer based on personality and demographics. The group attribute is race. Following~\cite{rezaei2021robust}, we used the dataset with 1885 samples and 11 features.
\end{itemize}

\subsection{Baselines}
\label{subsec:baselines}
We use the following baselines for comparison. This covers the exhaustive set of relevant methods described in section~\ref{sec:related_work} and Table~\ref{tab:methods}. \textbf{We reiterate that Table~\ref{tab:methods} is by no means intended to be an extensive survey, but rather a representative collection.}
 \begin{itemize}
     \item \textbf{MLP} is the standard \textit{Multi Layer Perceptron} classifier that doesn't take into account shift and fairness properties. In the standard in-distribution evaluation settings, such a model usually provides the upper bound to the accuracy without considering fairness, however the scenario differs as we are dealing with distribution shifts.
     \item \textbf{Adversarial Debiasing (AD)}~\citep{10.1145/3278721.3278779} is one of the most popular debiasing methods in the literature. This method performs well on the fairness metrics under the standard in-distribution evaluation settings, but fails to do so in the shift setting.
     \item \textbf{Robust Fair (RF)}~\citep{NEURIPS2020_d6539d3b} proposes a framework to learn classifiers that are fair not only with respect to the training distribution, but also for a broad set of distributions characterized by any arbitrary weighted combinations of the dataset. 
     \item \textbf{Robust Shift Fair (RSF)}~\citep{rezaei2021robust} is a recent and most relevant baseline to this work. The authors propose a method to robustly learn a classifier under covariate shift, with fairness constraints. A severe limitation of this method is that it requires explicit estimation of both source and target covariates' distributions. 
     \item \textbf{Z-Score Adaptation (ZSA)} following the thread of work under \textit{Batch Norm Adaptation}~\citep{li2017revisiting, NEURIPS2020_85690f81} literature, we implement a baseline that adapts the parameters of the normalizing layer by recomputing the z-score statistics from the unlabeled test data points.
     \item \textbf{Re-weighted Gradient Descent (RGD-Exp)} ~\citep{kumar2023stochastic} is a very recent work focusing on ditributional robustness where we add an additional "Wasserstein Representation Matching term", the 2nd term in eqn. \ref{eq:obj_interest}, so we are essentially replacing the first term in eqn. \ref{eq:obj_interest} by the loss mentioned in \citep{kumar2023stochastic}. We compare against this method to prove that simply adding the wasserstein representation matching term to SOTA distributional robustness methods does not suffice and something more novel is required to attempt to solve fairness under shift.
 \end{itemize}
 
The additional methods KLIEP and LISF described previously are implemented as follows: First, the importance ratio $\frac{d\PS(\rmX)}{d\PT(\rmX)}$ is estimated using unsupervised test samples and the training samples available based on the KLIEP and  LSIF losses (given in \cite{pmlr-v48-menon16}) via a parameterized weight network $s(X)$. Then, we train a classifier based on the following instance weighted cross entropy loss and representation matching loss:
      \begin{align}
        \min \limits_{(F_\theta= h_\theta \circ g_\theta )}  &\frac{1}{n} \sum_{(\rmX_i,\rmY_i,A_i) \in \TR} s(\rmX_i) \left(- \log P_{(F_\theta)}(\preds=Y_i|\rmX_i) \right)  +  \nonumber \\ &\lambda {\hat{\L}_{Wass}}(\TE) 
      \end{align}
    where $s(X)$ is a non-negative function which is obtained by minimizing:
    \begin{align}
        L_{\mathrm{KLIEP}} (s(\rmX)) = &\frac{1}{m} \sum_{\rmX_i \in {\cal D}^T} - \log s(\rmX_i) +  \nonumber \\ &\left(\frac{1}{n} \sum_{\rmX_i \in {\cal D}^S}s(\rmX_i) -1 \right)^2 
    \end{align}
      or,
        \begin{align}
        L_{\mathrm{LSIF}} (s(\rmX)) = &\frac{1}{m} \sum_{\rmX_i \in {\cal D}^T}  - s(\rmX_i) + \nonumber \\ & \frac{1}{2}  \left(\frac{1}{n} \sum_{\rmX_i \in {\cal D}^S} \left(s(\rmX_i) \right)^2 \right) 
    \end{align}

\subsection{Implementation Details}
\label{subsec:implementation_details}
We use the same model architecture across MLP and our method in order to ensure consistency. Following~\cite{wang2021equalized}, a \emph{Fully Connected Network} (FCN) with $4$ layers is used, where the first two layers compose $g$ and the subsequent layers compose $h$. For AD, we use an additional $2$ layer FCN that serves as the \emph{adversarial head} $a : g(\rmX) \rightarrow \Acal$ (similar to ~\citep{wang2021equalized}).

Without further specification, we use the following hyperparameters to train MLP, AD and ZSA. The number of epochs is set to $50$ with \emph{Adam} as the optimizer~\citep{https://doi.org/10.48550/arxiv.1412.6980} and weight decay of $1e^{-5}$ (for Adult dataset, the weight decay is $5e^{-4}$). The learning rate is set to the value of $1e^{-3}$ initially and is decayed to $0$ using the \emph{Cosine Annealing} scheduler~\citep{loshchilov2017sgdr}. A batch size of $32$ is generally used to train the models. The gradients are clipped at the value of $5.0$ to avoid explosion during training. The dropout~\citep{JMLR:v15:srivastava14a} rate is set to $0.25$ across the layers. For AD, the adversarial loss hyperparameter post grid search is used. 

RF and RSF works have tuned their model for the specific architecture and corresponding hyperparameters (different from the aforementioned specifics). We perform another grid search over these hyperparameters and report the best results for comparison.

For our proposed method, we pre-train the model for $15$ epochs with only $\ERS$. For the next $35$ epochs we use the objective in eq~\ref{eq:entropy_final_obj} (\textbf{Note}: eq~\ref{eq:entropy_final_obj} refers to the final equation in subsection \ref{subsec:weighted-ent-obj} in the main paper wherever referenced), but with a higher training data batch size to reduce variance in the \textit{Monte Carlo Estimation} of the second constraint  $\left(\frac{1}{n} \sum_{\rmX_i \in \TR} \frac{1}{\funcnew_w(g(\rmX_i))} = 1\right)$. The value of $m$ (size of $\TE$) is kept at $50$ for the main experiments, which is $<<$ size of $\TR$. The primary experiments are run with the shift magnitude $\gamma=10$ (with ablations provided in section~\ref{subsec:shift_magnitude}). The constraints $\mathcal{C}_1$ and $\mathcal{C}_2$ as mentioned in~\ref{eq:entropy_final_obj} are implemented as squared error terms where we minimize $ c_1 \cdot \left(\left(\frac{1}{m} \cdot \sum_{\rmX_i \in \TE} \funcnew_w(g(\rmX_i))\right) - 1\right)^2 + c_2 \cdot \left(\left( \frac{1}{n} \cdot \sum_{\rmX_i \in \TR} \frac{1}{\funcnew_w(g(\rmX_i))}\right) - 1\right)^2$, where $c_1$ and $c_2$ are hyperparameters to control the relative importance of each constraint. The values of the tuple ($\lambda_1$, $\lambda_2$) are set to the following - Adult : ($1, 0.01$) ; Arrhythmia : ($0.01, 0.005$) ; Communities : ($0.005, 0.0001$) and Drug : ($0.1, 0.1$) post grid search. 
\begin{figure}[t]
    \centering
    \includegraphics[width=80mm,height=40mm]{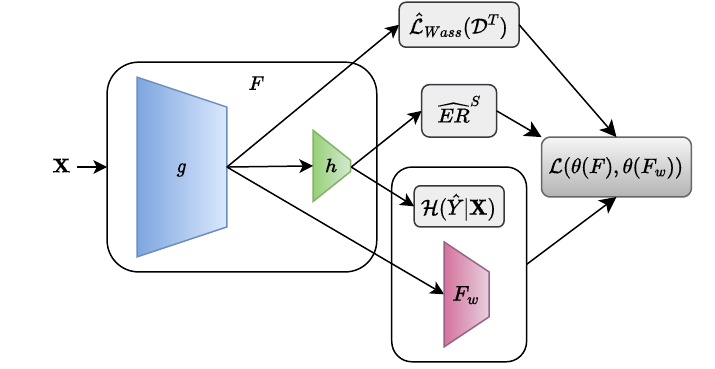}
    \caption{High level architecture of our method. Colored blocks represent parameterized sub-networks. }
    \label{fig:our_method}

\end{figure}
\subsection{Additional Experimentation}
We provide attached two additional figures: Figure \ref{fig:symmetric_fig_results_vars} corresponding to the detailed plot equivalent to Figure \ref{fig:symmetric_fig_results_pareto}. Here we highlight the high variance suffered by other baselines, especially the importance sampling ones, and reinforce the effectiveness of our method to achieve the best fairness-accuracy tradeoff. Figure \ref{fig:asymmetric_fig_results_2} corresponds to the case where the shift is across group 1. We refer you to Section \ref{sec:main_results_} and \ref{subsec:app_aymm_section} in the main paper respectively for the discussion.
\begin{figure*}[t]
    \centering
    \includegraphics[width=\textwidth]{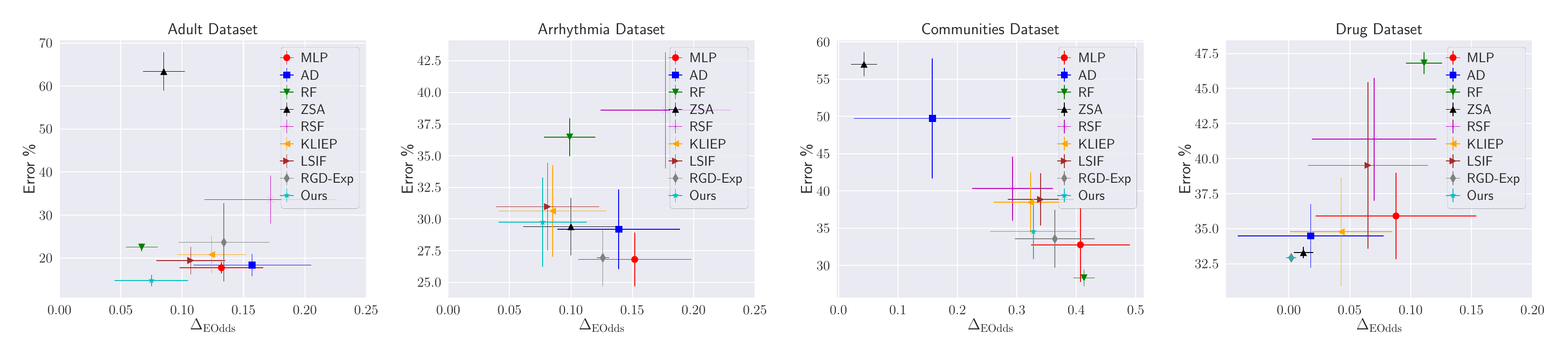}
    \caption{Comparison of our method against the baselines under Covariate Shift. The bars provide the standard deviation intervals both error (vertical) and $\EOdds$ (horizontal). }
    \label{fig:symmetric_fig_results_vars}
\end{figure*}
\begin{figure*}[t]
    \centering
    \includegraphics[width=\textwidth]{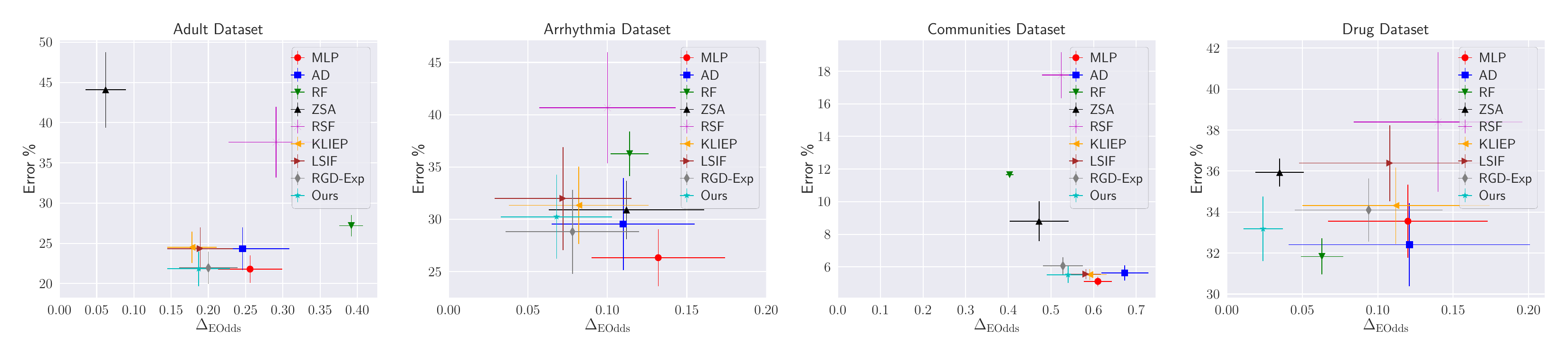}
    \caption{Comparison of our method against the baselines under Asymmetric Covariate Shift for group $\ermA = 1$}
    \label{fig:asymmetric_fig_results_2}
\end{figure*}

\section{Ablations}

\subsection{Unsupervised Adaptation with our Entropy Formulation Under Asymmetric Shift}
\label{subsec:our_entropy_asymm}
The asymmetric shift setup described in section~\ref{subsec:covar_shift} provides a well grounded motivation (section~\ref{sec:motivation}) and use case for explicitly handling shifts along with the unfairness objective. We further provide detail the rationale of the asymmetric covariate shift setup here. We work in a situation where training samples have labels while test samples do not. A1 and B1 (data from sub groups A (under-represented) and B (over-represented)) belong to training while A2 and B2 are the test data from the same subgroups. Because group B is covered well (many samples are available), it covers the domain of that subgroup quite well in the train itself which overlaps significantly with the test (see Fig. 1a). So a naive attempt at using representation matching to match the groups in the test (which will force B2 and A2 covariates to overlap in the transformed space) will end up concentrating on group B1 alone from the train since A1 need not map to A2 in the test. Therefore, a classifier built on this representation will mostly see only labels from B1 that map (roughly) to the area of A2. Our method overcomes these issues (elaborated in Section \ref{sec:Existing-Methods}) as we demonstrate empirically. 

We complement the claim with empirical evidence here. The results in table~\ref{tab:asymm_only_wass} provide comparison of the performance across the metrics with and without our proposed formulation on shifting group $A=0$. The wasserstein objective in section~\ref{sec:motivation} is retained in both settings. We observe significant improvements on both error and $\EOdds$ with our formulation. Particularly on the Drug dataset, we see an improvement of almost $4\%$ in error and around $13\times$ in the $\EOdds$, which is also notable on Arrhythmia.

\begin{table*}[t]
\small

\bigskip
\centering
\begin{tabular}{c|cc|cc}
Dataset          & \multicolumn{2}{c|}{Arrhythmia}                  & \multicolumn{2}{c}{Drug}                             \\ \hline
Entropy Variation  & \multicolumn{1}{c}{Without Entropy} & \multicolumn{1}{c|}{With Entropy (eq~\ref{eq:entropy_final_obj})} & \multicolumn{1}{c}{Without Entropy} & With Entropy (eq~\ref{eq:entropy_final_obj})   \\ \hline
Error \%                          & \multicolumn{1}{c}{28.648 \tiny(3.079)}  & 27.648 \tiny(2.978)                    & \multicolumn{1}{c}{35.859 \tiny(3.437)}  & 31.910 \tiny(0.186) \\
$\EOdds$                           & \multicolumn{1}{c}{0.080 \tiny(0.032)}   & 0.071 \tiny(0.037)                     & \multicolumn{1}{c}{0.076 \tiny(0.043)}   & 0.006 \tiny(0.013)  \\ %
\end{tabular}%
\caption{Comparison of the performance on using the unfairness objective without and with the unsupervised adaptation (our proposed entropy formulation). We observe substantial improvements in both error and $\EOdds$. Numbers in the parenthesis represent standard deviation across the 50 runs.}
\label{tab:asymm_only_wass}

\end{table*}

\subsection{Variation of $\lambda_1$ and $\lambda_2$}
\label{subsec:ablation_lambdas}
In this section, we study the variation of the performance of our method against the hyperparameters governing error ($\lambda_1$) and $\EOdds$ ($\lambda_2$). While studying the effect of either, we keep the other constant.

\begin{table*}[t]
\small

\bigskip
\centering
\begin{tabular}{c|cccccc}
$\lambda_1 \xrightarrow[]{} $ & 0              & 0.001          & 0.005          & 0.01           & 0.1            & 1.0            \\ \hline
Error (in \%)                  & 23.819 \tiny(8.593) & 22.047 \tiny(6.631) & 20.510 \tiny(6.706) & 20.851 \tiny(7.829) & 14.626 \tiny(1.318) & 14.787 \tiny(1.326) \\
$\EOdds$                   & 0.131 \tiny(0.038)  & 0.126 \tiny(0.037)  & 0.129 \tiny(0.037)  & 0.129 \tiny(0.029)  & 0.104 \tiny(0.033)  & 0.075 \tiny(0.30)   \\ %
\end{tabular}%
\caption{Variation of the performance of our method with Entropy Regularizer $\lambda_1$ on Adult dataset.}
\label{tab:change_lambda_1}

\end{table*}
Table~\ref{tab:change_lambda_1} reports the variation for $\lambda_1$ keeping $\lambda_2 = 0.01$ fixed. It is evident from the numbers that increasing $\lambda_1$ has strong correlation with the reduction in error, which exhibits a saturation at $0.1$. Higher values of $\lambda_1$ emphasize the minimization of the worst-case weighted entropy thus helping in calibration of the network in regions across $\PT$. Furthermore, we observe significant improvements in $\EOdds$ which is inline with the motivation of handling shifts along with an unfairness objective (section~\ref{sec:motivation}). Increasing $\lambda_1$ doesn't help post a threshold value as the correct estimation of the true class for a given $\rmX$ under $\PT$ becomes harder, particularly in regions far from the labeled in-distribution data. Imposing very strong $\lambda_1$ can hurt the model performance. 

The variation against $\lambda_2$, keeping $\lambda_1 = 1$ fixed is reported in table~\ref{tab:change_lambda_2}. As $\lambda_2$ increases, we observe a gradual improvement in $\EOdds$. This exhibits a maxima after which the performance degrades drastically. This is because strongly penalizing  ${\hat{\L}_{Wass}}(\TE)$ with a small number of samples $m$ leads to overfitting (illustrated by the large standard deviation) while matching $\PT(\rmX|\ermA)$. This also hurts the optimization as demonstrated by the substantial increase in error. 
\begin{table*}[t]

\centering
\begin{tabular}{c|cccccc}
$\lambda_2 \xrightarrow[]{} $ & 0              & 0.001          & 0.005          & 0.01           & 0.1             & 1.0             \\ \hline
Error (in \%)                  & 15.049 \tiny(1.424) & 15.849 \tiny(1.437) & 14.901 \tiny(1.352) & 14.787 \tiny(1.326) & 17.936 \tiny(15.962) & 42.280 \tiny(32.581) \\
$\EOdds$                    & 0.091 \tiny(0.031)  & 0.099 \tiny(0.034)  & 0.098 \tiny(0.032)  & 0.075 \tiny(0.030)  & 0.074 \tiny(0.036)   & 0.093 \tiny(0.064)   \\ %
\end{tabular}%
\caption{Variation of the performance of our method with Wasserstein Regularizer $\lambda_2$ on Adult dataset.}
\label{tab:change_lambda_2}
\end{table*}

\subsection{Shift Magnitude}
\label{subsec:shift_magnitude}
We study the variation of the performance of our method against the magnitude of shift $\gamma$ on Arrhythmia. A comparison against the best baseline ZSA is also provided. The variation of error is plotted in the left subfigure of~\ref{fig:shift_gamma}. With no shift in the data, $\gamma=0$, we observe that both the methods exhibit small errors as $\TE$ follows in-distribution. With the increase in the value of $\gamma$, ZSA shows a sudden increment in the error with an unstable pattern whereas our method exhibits a more gradual pattern and lower error as compared to ZSA. This justifies that the weighted entropy objective helps.
\begin{figure*}[t]
    \centering
    \includegraphics[width=\textwidth]{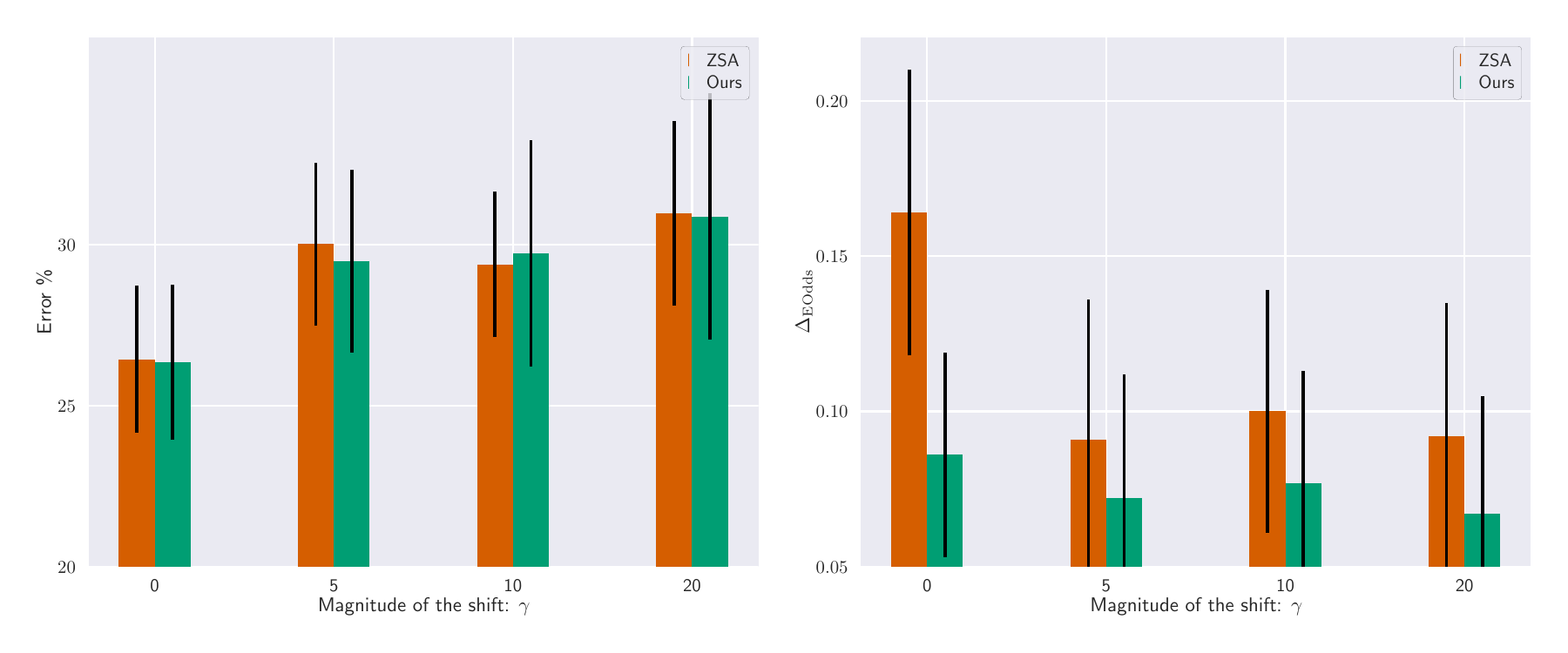}
    \caption{Variation of error against $\gamma$ (left subfigure) and $\EOdds$ against $\gamma$ (right subfigure) on Arrhythmia Dataset. We observe that our method performs better in both metrics against the best baseline ZSA. While the error increases gradually, but we observe substantially better $\EOdds$ for our method.}
    \label{fig:shift_gamma}
\end{figure*}

On the contrary, we observe that our method is highly stable over $\EOdds$ and performs consistently better for larger shifts as compared to ZSA. We attribute this effect to the proposed objective which optimizes the model to learn fairly under the shift and over the worst case scenario. %

\subsection{Variation of size of $\TE$}
\label{subsec:size_test}
Here, we study the dependence of the methods on the size of $\TE$. The left subfigure in~\ref{fig:k-shot} plots the variation of error against $m$. The error gradually decreases  for our method and RSF as the estimation of the true test distribution improves and the optimization procedure covers a larger region of $\PT$. This also makes the approximation by $\funcnew_w$ much more reliable and closer to true ratios. Although, the results don't show notable improvements after a certain threshold as we are dealing in an unsupervised regime over $\PT$. It becomes increasingly harder to correctly estimate the true class for a given $\rmX$ under $\PT$, particularly in regions far from the labeled in-distribution data. Interestingly, ZSA doesn't exhibit any improvements which demonstrates that merely matching first and second order moments across the data is not sufficient to handle covariate shifts.

The right subfigure in~\ref{fig:k-shot} plots the variation of $\EOdds$ against $m$. Here, we observe a consistent reduction in $\EOdds$ as more data from $\PT$ helps is matching representations via improved approximation of $\PT(\rmX|\ermA)$. Further this objective only deals with matching representations across the groups and doesn't stagnate as quickly with increasing $m$ as the error margins, which suffers from lack of reliable estimation in regions far from in-distribution.

We consistently outperform RSF in both very small and larger regimes of $m$, partly verifying the importance of $\funcnew_w$ rather than a direct estimation of $\PS$ and $\PT$ as RSF does. ZSA is substantially worse than both RSF and our method in terms of errors. In terms of $\EOdds$ its only marginally better than our method for $m=10$ and $m=20$, but at a huge expense of prediction performance.

\begin{figure*}[t]
    \centering
    \includegraphics[width=\textwidth]{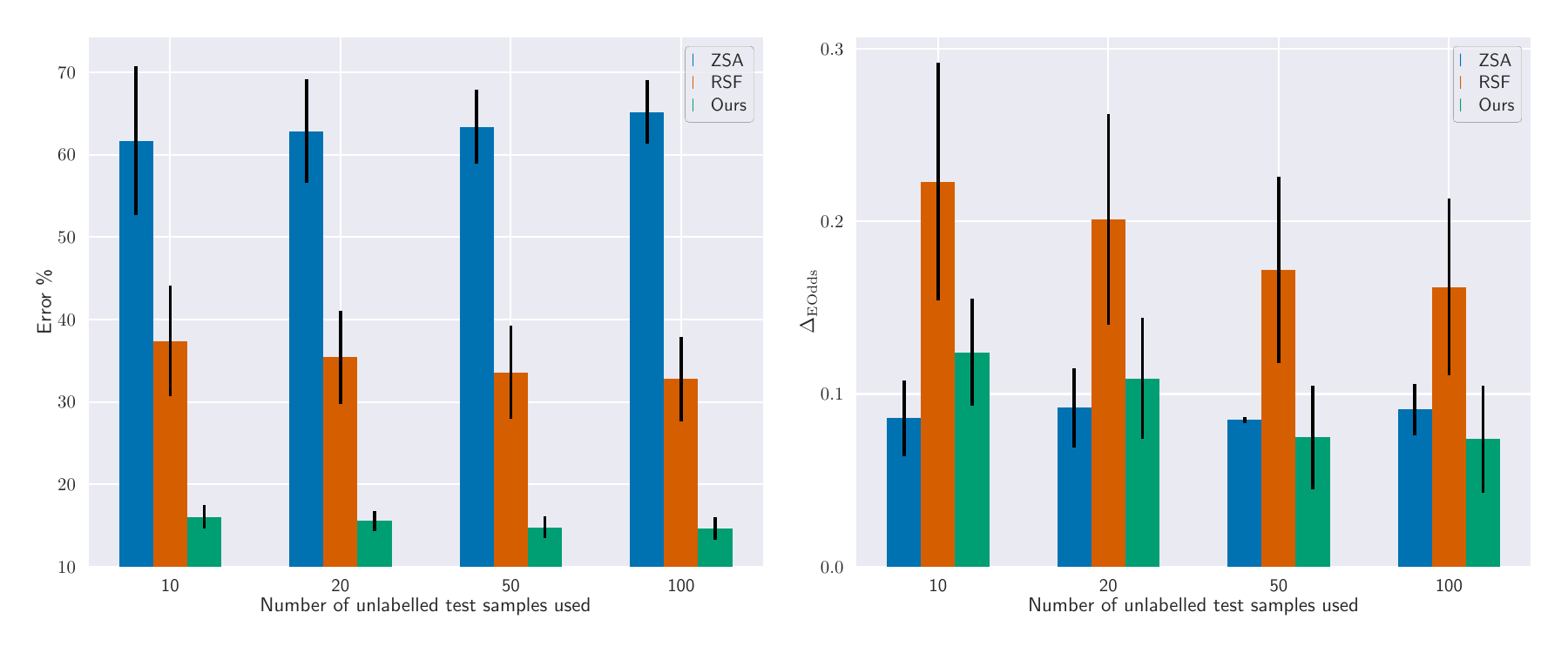}
    \caption{Variation of error against $m$ (left subfigure) and $\EOdds$ against $m$ (right subfigure) on Adult Dataset. We see reduction in both error and $\EOdds$ with increasing value of $m$.}
    \label{fig:k-shot}
\end{figure*}

\subsubsection{Unweighted Entropy vs Our Weighted Entropy Formulation}
\label{subsec:comparison_to_unweighted}
It is easy to observe that we can recover the standard unlabeled test entropy minimization using our derivation. Formally specifying, we can upper bound eq~\ref{eq:semifinal_ent} to obtain entropy as follows:
\begin{align}
    &\TRS + \epsilon \times \mathbb{E}_{\PT(\rmX)} \left[ e^{\left(- \frac{d\PS(\rmX)}{d\PT(\rmX)}\right)} \mathcal{H(\preds|\rmX)}  \right] < \nonumber \\ &\TRS + \epsilon \times \mathbb{E}_{\PT(\rmX)} \left[\mathcal{H(\preds|\rmX)}  \right], \because e^{-x} \leq 1, \forall x \geq 0
\end{align}
Our formulation particularly provides a tighter bound as compared to standard entropy and implicitly accounts for points in $\TE$ that are close to $\PS$ by assigning low weight. The experimental results comparing the two settings both with and without the unfairness objective are provided in table~\ref{tab:weight-unweight-ent}. Our formulation achieves substantially better results with a relative improvement of around $33\%$ in error. Note that due to the fairness-error tradeoff, the standard (unweighted) entropy achieves better $\EOdds$, but that is achieved at the expense of a nearly random classifier as evident from the error rate of nearly $50\%$. We also highlight the large standard deviation in the results achieved by unweighted entropy. This is largely because it seeks to minimize entropy across all $m$ points whereas our objective is more adaptive based on the approximation of importance ratio.%

\begin{table*}[t]
\small

\centering
\begin{tabular}{c|cc|cc}
 & \multicolumn{2}{c|}{Without Wasserstein Objective }                                        & \multicolumn{2}{c}{With Wasserstein Objective }                      \\ \hline
Entropy $\xrightarrow[]{}$     & \multicolumn{1}{c}{Unweighted} & \multicolumn{1}{c|}{Weighted (Ours)} & \multicolumn{1}{c}{Unweighted} & Weighted (Ours) \\ \hline
Error \%                 & \multicolumn{1}{c}{45.787 \tiny(12.900)}    & 34.291 \tiny(4.463)                        & \multicolumn{1}{c}{45.654 \tiny(13.090)}     & 35.549 \tiny(3.748)    \\
$\EOdds$                                & \multicolumn{1}{c}{0.204 \tiny(0.194)}      & 0.359 \tiny(0.074)                         & \multicolumn{1}{c}{0.201 \tiny(0.200)}      & 0.328 \tiny(0.073)    \\ %
\end{tabular}%
\caption{Comparison of the performance of Standard Unweighted Entropy v/s our Weighted Entropy formulation on Communities dataset, both with and without the Wasserstein Objective.}
\bigskip
\label{tab:weight-unweight-ent}
\end{table*}

\subsubsection{Weighting function on Inputs v/s Representations}
\label{subsec:weight_func_inputs}
The parametrized ratio estimator, $\funcnew_w$, takes as input the representations obtained from the subnetwork $g$ in eq \ref{eq:entropy_final_obj}. Here, we empirically compare and justify the use of this design choice. The experimental result in table \ref{tab:ratio_comp_rep} compares the two scenarios: (i) $\funcnew_w(\cdot)$ operating directly over the inputs and (ii) $\funcnew_w(g(\cdot))$  operating over the representation layer. From the numbers, we note that $\funcnew_w(g(\cdot))$ exhibits better fairness-error tradeoffs in general. Notably, we observe that on similar error values, the Equalized Odds are consistently better when using $\funcnew_w(g(\cdot))$. %

\begin{table*}[t]
\scriptsize
\centering

\resizebox{\textwidth}{!}{%
\bigskip
\begin{tabular}{c|cc|cc|cc|cc}
Dataset:            & \multicolumn{2}{c|}{Adult}                                            & \multicolumn{2}{c|}{Arrhythmia}                                       & \multicolumn{2}{c|}{Communities}                                      & \multicolumn{2}{c}{Drug}                                             \\ \hline 
Weighting:  & \multicolumn{1}{c}{$\funcnew_w(\rmX)$}        & \multicolumn{1}{c|}{$\funcnew_w(g(\rmX))$} & \multicolumn{1}{c}{$\funcnew_w(\rmX)$}        & \multicolumn{1}{c|}{$\funcnew_w(g(\rmX))$} & \multicolumn{1}{c}{$\funcnew_w(\rmX)$}        & \multicolumn{1}{c|}{$\funcnew_w(g(\rmX))$} & \multicolumn{1}{c}{$\funcnew_w(\rmX)$}        & \multicolumn{1}{c}{$\funcnew_w(g(\rmX))$} \\ \hline
Error \%                           & \multicolumn{1}{c}{ 14.926 \tiny(1.472)} & 14.787 \tiny(1.326)                  & \multicolumn{1}{c}{27.083 \tiny(1.899)} & 29.746 \tiny(3.519)                  & \multicolumn{1}{c}{34.078 \tiny(3.925)} & 34.549 \tiny(3.748)                  & \multicolumn{1}{c}{ 32.944 \tiny(0.180)} & 32.928 \tiny(0.143)                  \\
$\EOdds$                             & \multicolumn{1}{c}{ 0.095 \tiny(0.035)}  & 0.075 \tiny(0.030)                   & \multicolumn{1}{c}{0.134 \tiny(0.041)}  & 0.077 \tiny(0.036)                   & \multicolumn{1}{c}{0.369 \tiny(0.066)}  & 0.328 \tiny(0.073)                   & \multicolumn{1}{c}{0.003 \tiny(0.006)}  & 0.002 \tiny(0.004)                   \\ %
\end{tabular}%
}
\caption{Error-Fairness comparison for the two different settings of the ratio estimator network $\funcnew_w$. In the first setting, $\funcnew_w(\rmX)$, the raw feature vectors serve as the input whereas the setting leveraged for the experiments in this work, $\funcnew_w(g(\rmX))$, the inputs are the representations obtained from subnetwork $g$.}%
\label{tab:ratio_comp_rep}
\end{table*}

\subsubsection{Empirical investigation of the bound $\epsilon$ in Theorem~\ref{theorem:th_1}}
\label{subsec:eps_investigation}
We compare the ratio of the prediction probabilities for the classes ($y \in \{0,1\}$) on the validation set (which is not available during training to our algorithm) between a classifier trained only on the training set (Train) and a classifier trained only on the held-out test set (Test). 

We plot the ratios in figure~\ref{fig:eps_plot_2} with outliers removed. The subfigures (a),(b) demonstrate the ratio for the true class label for the samples. Subfigures (c),(d) demonstrate the ratio for class $y = 0$ and subfigures (e),(f) demonstrate the ratio for class $y = 1$. Correspondingly, in figure~\ref{fig:eps_plot} we plot the ratios with outliers. Note that atmost \textbf{4} points in every plot are outliers with $ratios > 5$. This empirically justifies that $\epsilon$ can be set \emph{not too high} with high probability except for a few outliers.

\begin{figure*}[t]
    \centering
    \includegraphics[width=\textwidth]{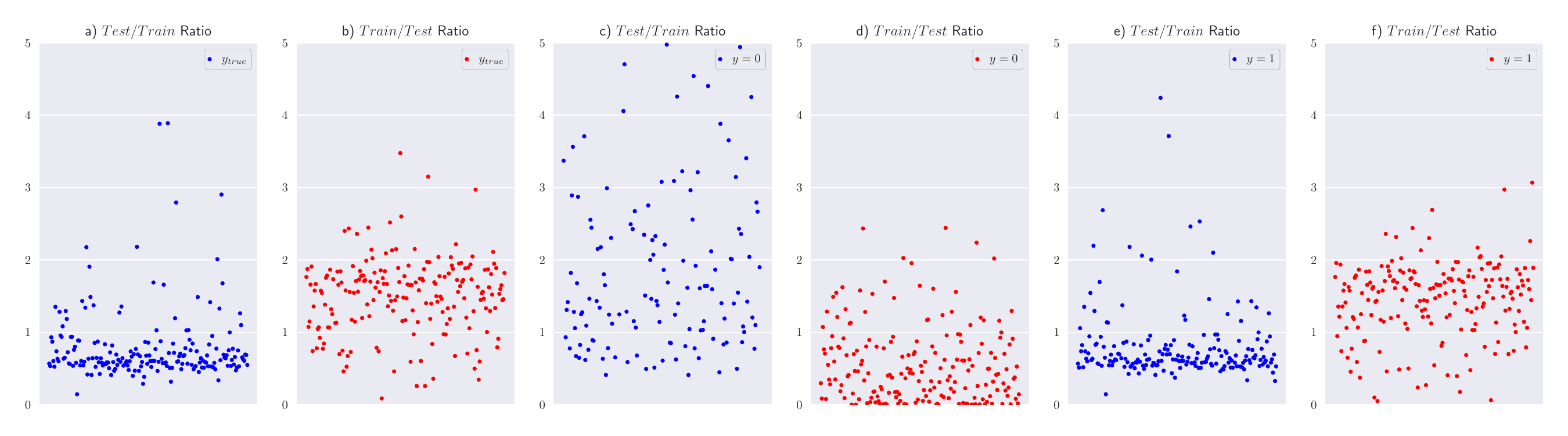}
    \caption{The subfigures demonstrate the ratio of the prediction probabilities for the classes ($y \in \{0,1\}$) on the validation set between a classifier trained only on the training set (Train) and a classifier trained only on the held-out test set (Test), with outliers removed. Note that $\epsilon=5$ provides a reasonable threshold and holds for all the samples but for 4 outliers (shown in figure~\ref{fig:eps_plot}).}
    \label{fig:eps_plot_2}

\end{figure*}

\begin{figure*}[t]
    \centering
    \includegraphics[width=\textwidth]{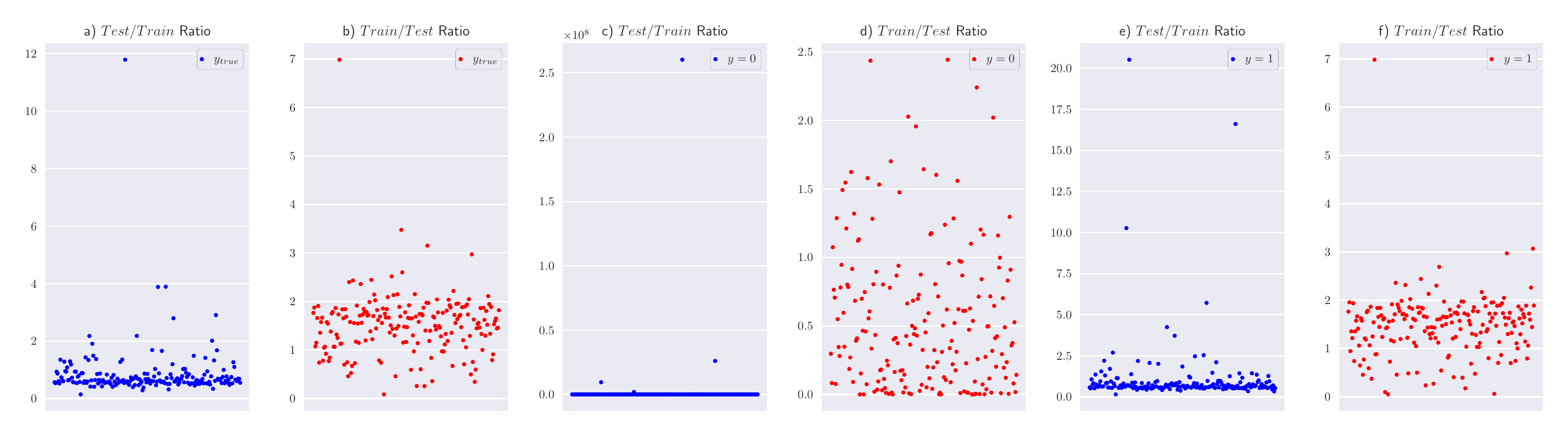}
    \caption{The subfigures demonstrate the ratio of the prediction probabilities for the classes ($y \in \{0,1\}$) on the validation set between a classifier trained only on the training set (Train) and a classifier trained only on the held-out test set (Test), with outliers. Atmost \textbf{4} points in every plot are outliers with $\mathrm{ratios} > 5$. }
    \label{fig:eps_plot}

\end{figure*}

\subsubsection{Comparison of Accuracy Parity}
\label{subsec:acc_parity_comp}
We further demonstrate the results of our method against the baselines with \emph{Accuracy Parity} as the fairness metric on Adult dataset in table~\ref{tab:acc_parity}. It is again evident from the results that our method clearly beats the baselines even on the accuracy parity front.
\begin{table*}[t]
\small
\centering

\bigskip

{%
\begin{tabular}{c|cccccccc}
Method -\textgreater{} & MLP    & AD     & RF     & RSF    & ZSA    & KLIEP  & LSIF   & Ours            \\ \hline
Error \%               & 17.735 & 18.356 & 22.525 & 33.591 & 63.396 & 20.787 & 19.428 & \textbf{14.787} \\ %
Accuracy Parity \%     & 5.764  & 3.156  & 6.094  & 11.851 & 11.131 & 4.417  & 6.677  & \textbf{2.990}  \\ \hline
\end{tabular}%
}
\caption{Comparison of Accuracy Parity as well as Error for all the methods. We outperform the baselines, particularly KLIEP and LSIF that are prone to poor results due to high variance.}
\label{tab:acc_parity}
\end{table*}

\subsubsection{Ratio estimated via $\funcnew_w(g(\rmX))$}
\label{subsec:ratio_f_g_x}
\begin{figure}[t]
    \centering
    \includegraphics[width=0.42\textwidth]{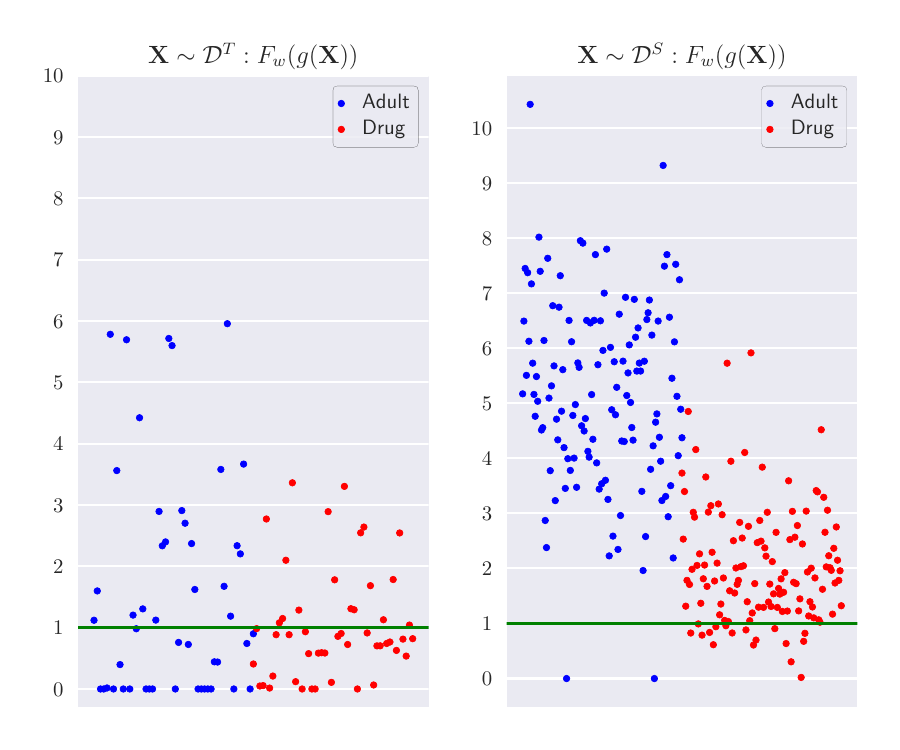}
    \caption{Comparison of the ratio estimated via $\funcnew_w(g(\rmX))$ across $\TE$ and $\TR$. The network learns the importance ratios w.r.t  $\PT$ and $\PS$.}
    \label{fig:weight_plots}

\end{figure}

We empirically justify the use of $\funcnew_w(g(\rmX))$ by comparing the distribution of the learned ratio across samples from $\TE$ and $\TR$ in figure~\ref{fig:weight_plots}. 

It is evident that the parametrized weight network can approximately learn importance ratios w.r.t  $\PT$ and $\PS$. The ratio computed for the test points lie mostly between $0$ and $1$ in order to satisfy $\mathcal{C}_1$ (in eq~\ref{eq:entropy_final_obj}) whereas the ratio computed for the training points are mostly $> 1$ in order to satisfy $\mathcal{C}_2$ (in eq~\ref{eq:entropy_final_obj}). More importantly, $\funcnew_w$ is learned end to end via optimization and doesn't incur any significant overhead compared to explicit density estimation.

\subsection{Proofs}
\label{subsec:gen_bound}

We recall some generalization bounds and definitions for finite hypothesis classes with bounded risks, that will will be used to prove theorems \ref{thm:IS_main} and \ref{thm:WE_main}.

For finite hypothesis classes, we recall generalization bounds for importance sampled  losses from
\cite{cortes2010learning}.

\begin{theorem}[\cite{cortes2010learning}] \label{cortes_thm}
Suppose that a dataset ${\cal D}$ is sampled i.i.d from distribution $\mathbb{P}(\rmX)$, $f_{\theta}$ is uniformly bounded by $L$ over the domain of $\mathbb{P}$, and a fixed weighing function $w(x)$ is such that $\sup w(x) = M,~ \mathbb{E}_{x \sim \mathbb{P}} [w(x)^2] \leq \sigma^2$. Consider the loss function $\tilde{l}(x) = w(x)f_{\theta}(x)$. We denote  $\tilde{l}=w \circ f_{\theta}$. then with probability $1-\delta$ over the draw of ${\cal D}$, $\forall \tilde{l} \in \{ w \circ f_{\theta}, ~{f_{\theta} \in \mathcal{F}_{\Theta}} \}$ we have:
\begin{align}
    \mathbb{E}_{x \sim \mathbb{P}} [\tilde{l}(x)] & \leq \sum \limits_{x \in {\cal D}} \tilde{l}(x) + \frac{2 M  (\log \lvert \Theta \rvert+ \log(1/\delta)) }{3 \mathcal{D}} + \nonumber \\ & L \sqrt{2\sigma^2 \frac{(\log \lvert \mathcal{F}_{\Theta} \rvert +  \log (1/\delta)) }{\lvert {\cal D}\rvert} }
\end{align}
\end{theorem}

We restate theorem \ref{thm:IS_main} here:%
\begin{theorem}\label{thm:IS_appendix}
Under Assumption \ref{assump}, with probability $1- \delta$ over the draws of $\TR \sim \PS$,  we have $\forall \theta \in \Theta$
\begin{align}
    \mathbb{E}_{\PS}[\mathcal{R}_{IS}(\theta)] \leq & \widehat{\mathcal{R}}_{IS}(\theta) + \frac{2 M (\log \lvert \Theta \rvert + \log (1/\delta)) }{3 \lvert \TR \rvert} + \nonumber \\
  &  \sqrt{2\sigma^2 \frac{(\log \lvert \Theta \rvert+  \log  (1/\delta)) }{ \lvert \TR \rvert} }
\end{align}
\end{theorem}
\begin{proof}
 The proof is a direct application of Theorem \ref{cortes_thm} to $R_{IS}(\theta)$ under Assumption \ref{assump}.
\end{proof}

\begin{definition}
Rademacher complexity ${\mathbf R}(A)$ for a finite set $A =\{a_1,a_2 \ldots a_N\} \subset \mathbb{R}^n$ is given by: 
  \begin{align}
      {\mathbf R}(A) = \mathbb{E}_{\sigma} [\sup \limits_{a \in A} \sum \limits_{i} \sigma_i a[i]  ]
  \end{align}
 where $\sigma$ is a sequence of $n$ i.i.d Rademacher variables each uniformly sampled from $\{-1,-1\}$ and $a[i]$ is the $i$-th coordinate of vector $a$.
\end{definition}

\textit{Empirical Rademacher complexity} of a class of finite number of functions $\mathcal{F}_{\Theta} = \{\funcnew_{\theta} | \theta \in \Theta \}$ on a data set ${\cal D}$ with $m$ samples is given by $\mathbf{R}(\mathcal{F}_{\Theta}({\cal D}))$ where $\mathcal{F}_{\Theta}({\cal D}) = \{ \mathrm{vec}(f_{\theta}(x),~\forall x \in D), ~ \forall \theta \in \Theta \} $.
Here, $\mathrm{vec}(\cdot)$ is a vector of entries.

\begin{theorem}[\cite{bousquet2003introduction}] $\mathbf{R}(\mathcal{F}_{\Theta}({\cal D})) \leq \left( \sup \limits_{f_{\theta} \in \mathcal{F}_{\Theta}, x \in {\cal D}} \lvert f_{\theta}(x) \rvert \right)  \sqrt{\frac{2 \log \lvert \mathcal{F}_{\Theta} \rvert}{ \lvert {\cal D} \rvert } }$.
\end{theorem}

\begin{theorem}\label{thm:rad_bound}
[\cite{bousquet2003introduction}] For a dataset ${\cal D}$  sampled i.i.d from distribution $\mathbb{P}(\rmX)$ and $f_{\theta}$ is uniformly bounded by $L, \forall \theta \in \Theta$ over the domain of $\mathbb{P}$, then with probability $1-\delta$ over the draw of ${\cal D}$, we have $\forall f_{\theta} \in \mathcal{F}_{\Theta}$

\begin{align}
    \mathbb{E}_{x \sim \mathbb{P}}[ f_{\theta}(x) ] & \leq \frac{1}{\lvert {\cal D} \rvert}\sum \limits_{x \in {\cal D}} f_{\theta}(x) + 2 \mathbf{R}(\mathcal{F}_{\Theta}({\cal D})) + \nonumber \\ &3 L \sqrt{ \frac{\ln(2/\delta)}{2 \lvert {\cal D} \rvert}}  \nonumber \\
    \hfill & \leq \frac{1}{\lvert {\cal D} \rvert}\sum \limits_{x \in {\cal D}} f_{\theta}(x) + 2 L \sqrt{\frac{ 2\log \lvert \mathcal{F}_{\Theta} \rvert}{ \lvert {\cal D}\rvert} }  + \nonumber \\ &3 L \sqrt{ \frac{\ln(2/\delta)}{2 \lvert {\cal D} \rvert}}
\end{align}
\end{theorem}

We restate theorem \ref{thm:WE_main} here:
\begin{theorem}
\label{thm:WE_appendix}
Under Assumption \ref{assump}, we have that with probability $1-2 \delta$ over the draws of $\TR \sim \PS$ and $\TE \sim \PT$, we have $\forall \theta \in \Theta$
\begin{align}
    \mathbb{E}_{\PS,\PT}[\mathcal{R}_{WE}(\theta)] \leq \widehat{\mathcal{R}}_{WE}(\theta) + 2  \sqrt{\frac{ 2\log \lvert \Theta \rvert}{ \lvert \TR \rvert} } + \nonumber \\
     2 \lambda \sqrt{\frac{ 2\log \lvert \Theta \rvert}{ \lvert \TE \rvert} }  +  3  \sqrt{ \frac{\ln(2/\delta)}{2 \lvert \TR \rvert}} +  3 \lambda \sqrt{ \frac{\ln(2/\delta)}{2 \lvert \TE \rvert}}
\end{align}
\end{theorem}

\begin{proof}
 We apply Theorem \ref{thm:rad_bound} to $\ell_1(\cdot)$ (which is bounded by $1$) and  $e^{-\tilde{w}(\cdot)}\ell_2(\cdot)$ where $\ell_2 (\cdot) \leq 1,~ e^{- \tilde{w}(\cdot)} \leq 1$ with the appropriate datasets in Assumption \ref{assump}. We then use union bound over the two error events that result from application of the theorem twice.
\end{proof}

\textbf{Key Takeaways:} As previously noted, by comparing Theorem \ref{thm:IS_appendix} and Theorem \ref{thm:WE_appendix}, we see that the generalization bound for importance sampled objective, $\mathcal{R}_{IS}$, depends on variance of importance weights $\sigma^2$ and also the worst case value $M$. In contrast, our objective, $\mathcal{R}_{WE}$, \textit{does not} depend on these parameters and thus does suffer from high variance. \\
Further, we note that $\mathcal{R}_{WE}$ depends on size of test set also the other apparently does not seem to. However, $\mathcal{R}_{IS}$ needs to estimate importance ratios - which will depend on the test set . As highlighted in the assumption \ref{assump}, we have analyzed both losses when the importance ratios are assumed to be known just to bring out the difference in dependencies on other parameters.

\textbf{Remark:} In Assumption \ref{assump}, we have assumed a finite hypothesis class $\Theta$. However, our result for Theorem \ref{thm:WE_appendix} would generalize (as is) with rademacher complexity or covering number based arguments of infinite functions classes. %
\cite{cortes2010learning} also point out analogous generalization for the importance sampling loss.  

\subsection{Representation Matching and Accuracy Parity}

We first quote a result from existing literature that bounds accuracy parity under representation matching.

\begin{theorem}[\cite{zhao2019inherent}]\label{thm:accparitybound}
Consider any soft classifier $\funcnew=h \circ g (\rmX) 
\in [0,1]$ and the hard decision rule $\preds = \mathbf{1}_{\funcnew(\rmX)>1/2}$. Let the Bayes optimal classifier for group $a$ under representation $g(\cdot)$ be: $\mathbf{1}_{\PT(\rmY=1|g(\rmX),A=a) > 1/2} =s_a(\rmX)$. Let the Bayes error for group $a$ under representation $g(\cdot)$ be $\mathrm{err}_a$.  Then we have:
\begin{align*}
 \Apar \leq & \overline{\mathbf{W}} + \lVert \PT(g(\rmX)| A=1) -  \PT(g(\rmX)| A=0) \rVert_1 \nonumber \nonumber \\  &+ \min \limits_a ( \mathbb{E}_{\PT(\rmX|a)}|s_1(\rmX)-s_0(\rmX)|)
\end{align*}
Here, $\overline{\mathbf{W}} = \sum \limits_a \mathrm{err}_a$ and  $\lVert \mathbb{P}(\cdot)- \mathbb{Q}(\cdot) \lVert_1$ is the total variation distance between measures $\mathbb{P}$ and $\mathbb{Q}$.
\end{theorem}

This suggests applying a loss for representation matching to enforce accuracy parity as it would drive the purely label independent middle term to zero.
However, we argue that, under \textit{asymmetric covariate shift} (Definition \ref{defn:asymmcov_shift}), accuracy parity is approximately the third term in Theorem \ref{thm:accparitybound} which is also large, even when the second term is set to $0$.

 Consider the covariate shift scenario given by Definition \ref{defn:asymmcov_shift}. Suppose one is also able to find a representation $g(\cdot)$ that matches across groups exactly in the test, i.e. $\PT(g(\cdot) | A=1)= \PT (g(\cdot) | A=0)$. %
Due to the asymmetric covariate shift assumption between train and test, we have $\PT(g(\cdot) | A=0) = \PT(g(\cdot) | A=1)= \PS (g(\cdot) | A=0)$. %
Since there is no covariate shift for group $A=0$, optimal scoring function $s_0(\rmX)$ remains the same even for the training set, given the representation.

Since a classifier $h$ is learnt on top of representation $g$, and only training distribution of group $A=0$ under $g$ overlaps (completely) with the test, classifier $h$ would be trained overwhelmingly with the correct labels for $A=0$ in the region where test samples are found. Over the test distribution, the hard decision score function will be approximately $s_0(\rmX)$. Therefore, the error in the group $0$ would be small. While the test error in group $1$ will be approximately $\mathbb{E}_{\PT(\cdot | A=1)} \left( |s_0(\rmX) - s_1(\rmX)| \right)$ which matches the third term in Theorem \ref{thm:accparitybound} and this third term is also large. %

\section{Limitations and Future Work}
Our work deals primarily with covariate shift where the label conditioned on covariates does not shift. In real world datasets, label shift often accompanies covariate shift. Our objectives were designed only to tackle covariate shift while it would be an interesting an future direction to extend this work to incorporate label shift also.

\end{document}